\documentclass{article}

\usepackage[utf8]{inputenc}
\usepackage[T1]{fontenc}
\usepackage{geometry}
\usepackage[svgnames]{xcolor}
\usepackage{tikz}
\usepackage{graphicx}

\usepackage{geometry}
\usepackage{graphicx, url, hyperref, amsmath, amssymb, mathtools, comment, xcolor}
\usepackage{amsthm, float} 
\usepackage[most]{tcolorbox}
\tcbuselibrary{skins, breakable}
\usepackage[shortlabels]{enumitem}
\usepackage{mathscinet} 

\usepackage{booktabs}
\usepackage{tabularx}
\usepackage{authblk}

\usetikzlibrary{shadings}
\tikzset{
  gradarrow/.style 2 args={
    shade,
    shading=axis,          
    shading angle=0,       
    left color=#1!30, 
    right color=#2!30,
    draw=none              
  }
}

\definecolor{StageA}{HTML}{0072B2}
\definecolor{StageB}{HTML}{009E73}
\definecolor{StageC}{HTML}{E69F00}
\definecolor{StageD}{HTML}{D55E00}

\theoremstyle{plain}
\newtheorem{theorem}{Theorem}
\newtheorem{lemma}{Lemma}

\theoremstyle{definition}

\newtheorem{remark}{Remark}[section]
\newtheorem{addendum}{Addendum}[section]

\newcommand{\Aletheia}{\emph{Aletheia}}

\definecolor{boxblue}{RGB}{0, 0, 150}
\definecolor{boxback}{RGB}{245, 245, 255}

\newtcolorbox{problem}[1]{%
    colback=boxback,
    colframe=boxblue,
    fonttitle=\bfseries\large,
    title={#1}, 
    sharp corners,
    enhanced,
    attach boxed title to top left={yshift=-2mm, xshift=2mm},
    boxed title style={colframe=boxblue, colback=boxblue},
    before skip=15pt plus 2pt, 
    after skip=15pt plus 2pt,
    top=10pt, bottom=10pt, left=10pt, right=10pt
}

\newtcolorbox{solution}[1]{%
    colback=white,
    colframe=boxblue,
    fonttitle=\bfseries\large,
    title={#1},
    sharp corners,
    enhanced jigsaw, 
    breakable,       
    attach boxed title to top left={yshift=-2mm, xshift=2mm},
    boxed title style={colframe=boxblue, colback=boxblue},
    before skip=15pt plus 2pt,
    after skip=15pt plus 2pt,
    top=10pt, bottom=10pt, left=10pt, right=10pt
}

\title{Semi-Autonomous Mathematics Discovery with Gemini:\\ A Case Study on the Erd\H{o}s Problems}

\author{
Tony Feng\textsuperscript{$\dagger$*}, 
Trieu Trinh\textsuperscript{*}, 
Garrett Bingham\textsuperscript{*}, 
Jiwon Kang\textsuperscript{$\dagger$}, 
Shengtong Zhang\textsuperscript{$\dagger$}, 
Sang-hyun Kim\textsuperscript{$\dagger$}, 
Kevin Barreto\textsuperscript{$\dagger$}, 
Carl Schildkraut\textsuperscript{$\dagger$}, 
Junehyuk Jung\textsuperscript{$\dagger$}, 
Jaehyeon Seo\textsuperscript{$\dagger$}, 
Carlo Pagano\textsuperscript{$\dagger$}, 
Yuri Chervonyi\textsuperscript{*}, 
Dawsen Hwang\textsuperscript{*}, 
Kaiying Hou\textsuperscript{$\dagger$}, 
Sergei Gukov\textsuperscript{$\dagger$}, 
Cheng-Chiang Tsai\textsuperscript{$\dagger$}, 
Hyunwoo Choi\textsuperscript{$\dagger$}, 
Youngbeom Jin\textsuperscript{$\dagger$}, 
Wei-Yuan Li\textsuperscript{$\dagger$}, 
Hao-An Wu\textsuperscript{$\dagger$}, 
Ruey-An Shiu\textsuperscript{$\dagger$}, 
Yu-Sheng Shih\textsuperscript{$\dagger$},
Quoc V. Le\textsuperscript{$\diamond$}, Thang Luong\textsuperscript{$\diamond$} \\

\small \textsuperscript{$\dagger$}Mathematical contribution, \textsuperscript{*}Engineering contribution, \textsuperscript{$\diamond$}Principal Investigators \newline
}

\date{\today}

\begin{document}

\maketitle

\insert\footins{\noindent\footnotesize 
\textit{Corresponding authors: fengt@berkeley.edu, thangluong@google.com.} \newline
\textit{Affiliations:} \textit{Google DeepMind} (Tony Feng, Trieu Trinh, Garrett Bingham, Junehyuk Jung, Yuri Chervonyi, Dawsen Hwang, Kaiying Hou, Sergei Gukov,  Quoc V. Le, Thang Luong), \textit{UC Berkeley} (Tony Feng), \textit{Seoul National University} (Jiwon Kang, Hyunwoo Choi, Youngbeom Jin), \textit{Stanford University} (Shengtong Zhang, Carl Schildkraut), \textit{Korea Institute for Advanced Study} (Sang-hyun Kim), \textit{University of Cambridge} (Kevin Barreto), \textit{Brown University} (Junehyuk Jung), \textit{Yonsei University} (Jaehyeon Seo), \textit{Concordia University} (Carlo Pagano), \textit{Caltech} (Sergei Gukov), \textit{Academia Sinica} (Cheng-Chiang Tsai), \textit{National Taiwan Unversity} (Wei-Yuan Li, Hao-An Wu, Ruey-An Shiu, Yu-Sheng Shih).}

\begin{abstract}
    We present a case study in semi-autonomous mathematics discovery, using Gemini\footnote{More precisely, the project was conducted using a math research agent built upon Gemini Deep Think, codenamed {\color{blue}\href{https://github.com/google-deepmind/superhuman/tree/main/aletheia}{Aletheia}}, introduced in \cite{aletheia}.} to systematically evaluate 700 conjectures labeled `Open' in Bloom's Erd\H{o}s Problems database. We employ a hybrid methodology: AI-driven natural language verification to narrow the search space, followed by human expert evaluation to gauge correctness and novelty. We address 13 problems that were marked `Open' in the database: 4 through seemingly novel autonomous solutions, and 9 through identification of previous solutions in the existing literature. Our findings suggest that the `Open' status of the problems resolved by our AI agent can be attributed to obscurity rather than difficulty. We also identify and discuss issues that arise in applying AI to math conjectures at scale, highlighting the difficulty of literature identification and the risk of ``subconscious plagiarism'' by AI. We reflect on the takeaways from AI-assisted efforts on the Erd\H{o}s Problems.
\end{abstract}

\tableofcontents

\section{Introduction}
Paul Erdős, among the most prolific mathematicians of the 20th century, left a vast legacy of papers and unsolved conjectures. In 2023, Thomas Bloom launched \href{https://www.erdosproblems.com/}{\ttfamily ErdosProblems.com}, a centralized repository designed to catalog these conjectures and track progress on them. At the time of this writing, the database tracks 1,179 problems, with 483 (41\%) classified as solved.

We stress, however, that the ``Open'' status of a problem in this database does not always reflect the true state of the literature. To quote from \cite{bubeck2025earlyscienceaccelerationexperiments}, ``in practice, a problem being listed as `open' roughly indicates that at least 1 professional mathematician attempted and failed to find a previously published solution by searching the internet.'' This gap became evident in October 2025, when OpenAI announced that GPT-5 identified ten ``Open'' problems on the website that had, in fact, already been resolved in the literature. A sharp increase in attention to Bloom's database followed, leading to further AI-related progress and prompting the recent creation of a community wiki by Terence Tao \cite{tao2026erdosai} to comprehensively track AI-assisted developments on the Erd\H{o}s problems.

This paper represents a case study of applying AI at scale to Bloom's Erd\H{o}s problem database. Large Language Models can easily generate candidate solutions, but the number of experts who can judge the correctness of a solution is relatively small, and even for experts, substantial time is required to carry out such evaluations. In particular, it would have been infeasible for our team to evaluate model outputs on all of the problems marked `Open'. We therefore used AI-based natural language verifiers to narrow the search space to a tractable scale for a small team of human experts. 
See Table~\ref{table:aletheia_erdos_solutions} for a synopsis of our results. A comparison to other known AI-assisted results (at the time of this writing) in Table~\ref{table:LLM_solutions}, and an explanation of the classification is in \S\ref{ssec:results}.

\begin{table}[ht]
  \centering

  \begin{tabularx}{\textwidth}{
    @{}
    >{\raggedright\arraybackslash}p{3cm}
    X
    >{\raggedright\arraybackslash}p{2.8cm}
    @{}
  }
    \toprule
    \textbf{Classification} & \textbf{Description} & \textbf{Instances} \\
    \midrule

    Autonomous Resolution &
    Autonomous novel solution. &
    \mbox{\textbf{652}\textsuperscript{*}}, \mbox{\textbf{1051}} \\
    \addlinespace

    Partial AI Solution &
    Solved some part of a multi-part problem. &
    \mbox{\textbf{654}}, \mbox{\textbf{1040}} \\
    \addlinespace

    Independent Rediscovery &
    Found a correct solution later discovered to exist in the literature. &
    \mbox{\textbf{397}\textsuperscript{*}}, \mbox{\textbf{659}\textsuperscript{*}}, \mbox{\textbf{935}}, \mbox{\textbf{1089}} \\
    \addlinespace

    Literature Identification &
    Identified that the problem was already solved in the literature. &
    \mbox{\textbf{333}\textsuperscript{*}}, \mbox{\textbf{591}\textsuperscript{*}}, \mbox{\textbf{705}}, \newline
    \mbox{\textbf{992}}, \mbox{\textbf{1105}} \\
    \bottomrule
  \end{tabularx}
  \caption{Taxonomy of {\Aletheia} results on Erd\H{o}s problems. \textsuperscript{*}Independently obtained by other parties 
  after our initial evaluations were conducted, but prior to the publishing of this work.}\label{table:aletheia_erdos_solutions}
\end{table}

An alternative approach to the evaluation problem is via formal verification, such as through the Lean language. This has also led to a handful of success cases, but has limitations. First, because only a tiny proportion of the math research literature is formalized in Lean, this significantly restricts the model's toolkit for solving problems. Second, many problem statements in the database are open-ended or susceptible to misinterpretation. An expert is still required to interpret the argument in natural language to determine if a formally verified proof addresses the intended mathematical meaning (see Appendix \ref{app:75} for an interesting case study on this issue). 


\subsection{A semi-autonomous effort based on Gemini Deep Think}\label{ssec:intro-results}
\textbf{Aletheia: a specialized math research agent.} From December 2--9 (2025), we deployed a custom mathematics research agent built upon Gemini Deep Think, internally codenamed {\Aletheia} at Google DeepMind \cite{aletheia}, on the then-700 Erd\H{o}s problems still marked as `Open' in Bloom's database. Crucially, {\Aletheia} includes a (natural language) verifier mechanism\footnote{This is the reason for the name ``Aletheia'', which is a homage to the Greek goddess of Truth.} that helped narrow the pool of problems to examine: from the original 700 problem prompts, 212 responses came back as potentially correct.

\textbf{Human evaluation.} A team of human mathematicians then filtered these responses to eliminate incorrect solutions. Most team members were not experts in the relevant problem domain, so we prioritized narrowing the pool of candidate solutions quickly (possibly at the cost of making noisier judgments) to a manageable scale for our smaller core of domain experts. After this step, there were 27 solution candidates to focus on. Then our internal domain experts vetted those candidates carefully, consulting external experts when correctness was ascertained but novelty was unclear. Our ultimate findings were that 63 solutions were technically correct, but only 13 solutions correctly addressed the \emph{intended} problem statement (either by invoking the literature, or by a novel argument).

\begin{remark} For these ``correct'' solutions, there were sometimes minor inaccuracies or omissions. Because of this, we found informal verification to be slightly subjective, even when performed by human experts.\footnote{The original output on Erdos-1051 from our initial December 2-9 sweep was graded by 5 mathematicians, 4 of which considered it to be already correct (within similar bounds of fixing minor accuracies as in the rest of the paper), while 1 marked it incorrect. The controversial issue was an important deduction, described by the model as ``standard comparison theorems for linear recurrences imply...'', which was viewed by four mathematicians as sufficiently standard to be acceptable, and the fifth as too unrigorous. Agreement was not reached even after internal debate, so when writing this paper we used for Erdos-1051 (only) the output of another model run which had been made for ablation purposes on the subset of meaningfully correct solutions. For transparency, both the original output and later ablation output are included in the raw output log.} 
\end{remark}

The following figure summarizes our process. 

\begin{figure}[ht]
    \includegraphics[width=\linewidth]{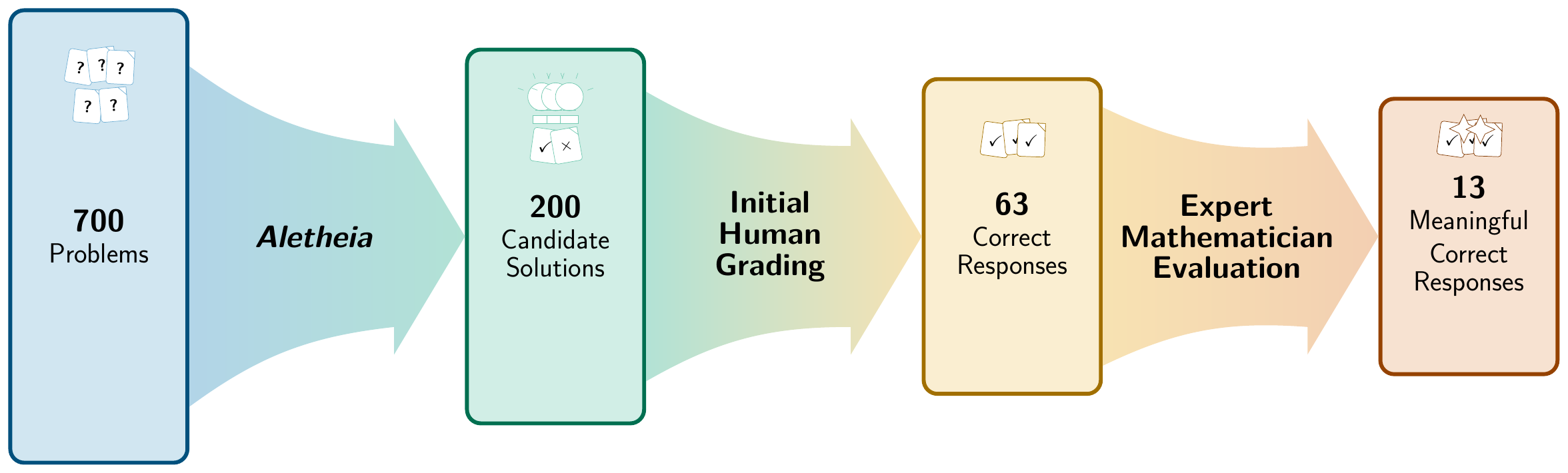}
\end{figure}

The remaining 50 of \Aletheia's correct solutions were technically valid but mathematically vacuous: the problem statements were interpreted in a way that did not capture Erd\H{o}s's intent, often (but not always) leading to trivial solutions. For these problems, our team will propose revised statements. To the wider community interested in Erd\H{o}s problems, we caution that even after correctly solving an Erd\H{o}s problem, one should take care to ensure the statement accurately reflects what Erd\H{o}s likely intended (this issue is discussed further below). 

Finally, 12 of the responses were marked ambiguously, for example due to open-endedness of the question itself.

In summary, out of the 200 solution candidates that we definitively marked correct or incorrect, 137 ($68.5\%$) of responses were fundamentally flawed, while 63 ($31.5\%$) of responses were technically correct, but only 13 ($6.5\%$) were meaningfully correct. 

\begin{table}[h]
\centering
\begin{tabular}{lrr}
\toprule
\textbf{Category} & \textbf{Count} & \textbf{Percentage} \\
\midrule
Fundamentally Flawed & 137 & 68.5\% \\
Technically Correct  & 63  & 31.5\% \\
\hspace{1em} \textit{Meaningfully Correct (subset)} & \textit{13} & \textit{6.5\%} \\
\midrule
\textbf{Total Candidates} & \textbf{200} & \textbf{100.0\%} \\
\bottomrule
\end{tabular}
\caption{Solution accuracy on 200 AI-generated responses, as graded by humans.}
\label{tab:solution-candidates}
\end{table}

\textbf{Transparency.} On the advice of Terence Tao, we emphasized the figures above for transparency. Although this was not the context for Tao's comments, the figures seem relevant to the common assertion that AI is ``accelerating science''. Without commenting on the validity of the claim, we point out that the evidence presented in its favor often has a one-sided nature. In the context of mathematics, the point is often argued by presenting only positive cases, where AI accomplishes a particular task faster than a human would have, and thereby ``accelerates'' that specific result. However, this does not account for negative cases that may involve considerable time spent checking AI-generated material for correctness, nor time spent debugging subtle AI-introduced errors, nor---even in the event of mathematically correct output---time spent searching the literature for possible AI plagiarism. Nevertheless, the authors of this paper are optimistic that the balance of these considerations will trend more positive in time. Our aim is to present a more complete perspective of both the strengths and the weaknesses of AI, so that they may be better addressed. \\

\textbf{New challenges.} Perhaps surprisingly, \textbf{the lengthiest and most arduous step} of our effort was the final one of investigating whether the solutions were already in the literature, and whether they really addressed the \emph{intended} problem. Some question formulations were eventually found to have very subtle issues, tracing back to mistranscriptions or omissions in either the website \emph{or in the original writings of Erd\H{o}s}, rendering the problems too easy. Most of the time, however, it was due to notational/definitional convention ambiguity, as {\Aletheia} had not been informed of the definitional conventions laid out on Bloom's site, and so would commonly confuse different (valid) interpretations of technical terms\footnote{e.g.,~additive versus Dirichlet convolution, strong versus weak completeness, etc.}.

Indeed, the number `13' of meaningfully correct solutions was substantially higher before we undertook this final investigation (and the number `4' of novel autonomous solutions was once as high as `9'), and it fell further after we circulated our solutions privately among external experts\footnote{(Added Feb 5, 2026) and again after public dessemination}; all decreases came from issues of disentangling literature rather than issues of mathematical correctness. \emph{Future AI-based efforts will need to be cautious in this regard.} In Appendix \ref{app:75}, we document the case of Erd\H{o}s-75: {\Aletheia} devised a correct solution to a non-trivial problem; however, after consulting external experts we discovered that the problem as listed on \href{https://www.erdosproblems.com/}{\ttfamily ErdosProblems.com} was \emph{not} the intended formulation; even though it was accurately transcribed from a paper \cite{Erd95} by Erd\H{o}s, \emph{Erd\H{o}s's own formulation was itself flawed} there. \\

 \subsection{Results}\label{ssec:results} Our 13 positive results clustered naturally into four categories which we felt should be distinguished.
\begin{description}
    \item[Autonomous Resolution.] On these problems, {\Aletheia} found the first correct solution (as far as we can tell) in a mathematically substantive way. These include  \textbf{Erd\H{o}s-652} and \textbf{Erd\H{o}s-1051}, although we note that \textbf{Erd\H{o}s-652} is solved by immediate reduction to results from the existing literature.
    \item[Partial AI Solution.] On these problems, there were multiple questions and {\Aletheia} found the first correct solution to one of the questions. These include \textbf{Erd\H{o}s-654}, and \textbf{Erd\H{o}s-1040}. 
    \item[Independent Rediscovery.] On these problems, {\Aletheia} found a correct solution, but human auditors subsequently found an independent solution already in the literature. These include \textbf{Erd\H{o}s-397}, \textbf{Erd\H{o}s-659}, \textbf{Erd\H{o}s-935}, and \textbf{Erd\H{o}s-1089}. They \emph{appear} to have been independently rediscovered by our model: we scanned the logs of \Aletheia's reasoning trace to ensure that the solution was not pulled \emph{directly} from the literature solution. It is of course possible that the solution was \emph{indirectly} ingested from the literature solution, either implicitly through intermediate sources or during training. This highlights a new danger that accompanies AI-generated mathematics: it is susceptible to ``subconscious plagiarism'' by reproducing knowledge acquired during training, without attribution. 
    \item[Literature Identification.] On these problems, {\Aletheia} found that a solution was already explicitly in the literature, despite the problem being marked ``Open'' on Bloom's website at the time of model deployment. These include \textbf{Erd\H{o}s-333}, \textbf{Erd\H{o}s-591}, \textbf{Erd\H{o}s-705}, \textbf{Erd\H{o}s-992}, \textbf{Erd\H{o}s-1105}. 
\end{description}
To be clear, \emph{we make no claims of novelty for the latter two categories}. The `4' autonomous solutions cited above refer to \textbf{Erd\H{o}s-652}, \textbf{Erd\H{o}s-654}, \textbf{Erd\H{o}s-1040}, and \textbf{Erd\H{o}s-1051}. In the estimation of our experts, none of the four individually rises to the level of a research paper. In fact, some of them are at the level of student exercises (given the existing literature). 

We tentatively believe {\Aletheia}'s solution to Erd\H{o}s-1051 represents an early example of an AI system autonomously resolving a slightly non-trivial open Erd\H{o}s problem of somewhat broader (mild)  mathematical interest, for which there exists past literature on closely-related problems \cite{KolouchNovotny2016DiophantineApproximations}, but none fully resolves Erd\H{o}s-1051. Moreover, it does not appear to us that {\Aletheia}'s solution is directly inspired by any previous human argument (unlike in many previously discussed cases), but it does appear to involve a classical idea of moving to the series tail and applying Mahler's criterion. The solution to Erd\H{o}s-1051 was generalized further, in a collaborative effort by {\Aletheia} together with human mathematicians and Gemini Deep Think, to produce the research paper \cite{BKKKZ}.

\subsection{The writing of this paper} The solutions presented in this paper are \emph{human-rewritten} versions of \Aletheia’s raw outputs. While we aimed to preserve the original style and format, we refined the prose to better suit academic standards\footnote{For eample, the model's outputs tended to be far more verbose than one would find in a mathematics journal paper.}, corrected minor inaccuracies, and consolidated references into a unified bibliography. Importantly, the \textbf{core mathematical logic of all the solutions remains unchanged.} 

Where the original outputs contained errors, we have provided remarks to explain those specific inaccuracies. For full transparency, the unedited\footnote{other than formatting them for {\LaTeX} compilation in a unified document, which we did automatically by prompting Gemini 3.0 with the directive, ``Format this text for compilation in a latex doument''.} raw outputs will be uploaded {\color{blue}\href{https://github.com/google-deepmind/superhuman/tree/main/aletheia}{here}}.

Our current belief is that \emph{mathematics papers should always be authored by humans, even when the AI-generated content is (fully) mathematically correct.} As a general principle, authorship in mathematics entails accountability for both mathematical validity and expositional integrity, such as the correctness of attributions, which is a responsibility that only humans can bear. However, this paper is based on a substantial amount of AI-generated mathematical text, and we take accountability for the material presented here.

\subsection{Contextualizing the results}\label{ssec:contextualizing} A disclaimer is necessary regarding the novelty of these results on Erd\H{o}s problems. While we made considerable efforts to review the literature, it is certainly possible that we missed earlier solutions to these problems by human mathematicians. Therefore, our initial classification into categories is, at best, an upper bound on novelty. It is subject to revision after further investigation by the public.\footnote{After the initial posting of this paper, we updated it with ``Addendums'' below problems where appropriate.} Indeed, previous AI-assisted work on Erd\H{o}s problems 1026, 397, 333, and 281 was discovered, after initial announcements of novelty, be redundant with the literature\footnote{For Erd\H{o}s-281, we note that the AI solution is distinct from the previously existing literature solution.}. To the outside observer, this may present a misleading impression of mathematics research: in the authors' experience, it is very unusual for human-generated results to be redundant in this manner (in the modern era of communication). One reason why it seems to be happening so frequently with AI-generated work on Erd\H{o}s problems is that \textbf{the solutions are so simple that they would not attract attention if they originated from humans.} For instance, Erd\H{o}s-1089 is answered by an offhand remark in a 1981 paper \cite{Bannai81}, where the authors seemed unaware that they had resolved an Erd\H{o}s problem. 

In fact, for \emph{all} of the AI-generated solutions which have not yet been located in the literature, we find it plausible that they were also discovered before by humans years ago (perhaps implicitly, as special cases of more general theorems), but were never published because they were not considered important enough. 
	
For this case study, we waited to conduct due diligence and gather a complete picture rather than releasing individual results one-by-one as we confirmed them. During that time, some of the problems were independently solved by other parties. For example, Erd\H{o}s-333 briefly garnered attention on social media for being ``solved'' by GPT-5.2 Pro, but our team quickly corrected this misconception, as {\Aletheia} had already identified that the problem was already solved in the literature. Later, the same (simple) counterexample to Erd\H{o}s-397 that {\Aletheia} found was independently discovered by a combination of GPT-5.2 Pro and Harmonic's \emph{Aristotle}. Even later, Erd\H{o}s-397 was discovered to be a variant of Problem 3 from Day 1 of the 2012 Chinese Team Selection Test for the IMO \cite{AoPS12}. While the date of our solution can be verified by internal logs, we are content to cede priority; indeed, \textbf{our takeaway from this experience is that resolving open Erd\H{o}s problems can be completely elementary}, depending on the problem. We stress that the \emph{mathematical significance of such resolutions can only be accurately evaluated by expert mathematicians}, even if the correctness can be ascertained by non-mathematicians or formal verifiers.

\subsection{Discussion of other AI results on Erd\H{o}s problems}
We briefly survey other autonomously resolved Erd\H{o}s problems. Despite significant social media ``hype'', many of these solutions proved to be derivative upon closer inspection. For example, as explained above, Erdős-397 was found to be nearly identical to a training problem from a Chinese Math Olympiad Team Selection Test (Remark \ref{rem:397}). Another instance of recent AI-adjacent work (though not one where AI played a role in finding\footnote{In either the usual sense of devising the argument, \emph{or} in the sense of locating one in the literature.} the solution) can be found in \cite{alexeev2025forbidden}, which explains that Erd\H{o}s-707 is disproved by an offhand example of Marshall Hall \cite{Hall47}, decades before Erd\H{o}s-707 was even posed.

\begin{table}[H]
    \centering
    \begin{tabularx}{\textwidth}{@{} >{\raggedright\arraybackslash}p{4cm} X @{}}
    \toprule
    \textbf{Classification} & \textbf{Instances} \\ \midrule
    Autonomous Resolution & 205\textsuperscript{*}, 281\textsuperscript{*}, 401\textsuperscript{*}, 543\textsuperscript{*}, \textbf{652}\textsuperscript{*}, 728\textsuperscript{*}, 729\textsuperscript{*},  \textbf{1051} \\ \addlinespace
    Partial AI Solution & \textbf{654}, \textbf{1040} \\ \bottomrule
    \end{tabularx}
    \caption{Taxonomy of all novel autonomous LLM results on Erd\H{o}s problems at the time of this writing. (Independent Rediscovery and Literature Identification not counted.) Numbers in bold were discovered by the effort documented in this paper. \textsuperscript{*}Independently obtained by other parties after our initial evaluations were conducted, but prior to the publishing of this work.}
    \label{table:LLM_solutions}
\end{table}

One of our authors (K.~Barreto) was in part responsible for the results on Erd\H{o}s-205, 401, 728, and 729, and stresses that the arguments closely follow prior human arguments. More specifically, GPT-5.2 Pro's solution to Erd\H{o}s-205 appears similar in spirit to a heuristic argument of Wouter van Doorn (user ``Woett'') on the corresponding site thread, and its solutions to Erd\H{o}s problems 401, 728, and 729 appear inspired by previous work of Pomerance. Indeed, Pomerance later explored these in \cite{Pomerance26}.

\subsection{Conclusions} Our results indicate that there is low-hanging fruit among the Erd\H{o}s problems, and that AI has progressed to be capable of harvesting some of them. While this provides an engaging new type of mathematical benchmark for AI researchers, \textbf{we caution against overexcitement about its mathematical significance}. Any of the open questions answered here, perhaps with the exception of Erd\H{o}s-1051, could have been easily dispatched by the right expert. On the other hand, the time of human experts is limited. AI already exhibits the potential to accelerate attention-bottlenecked aspects of mathematics discovery, at least if its reliability can be improved.

In our case study, we encountered difficulties that were not anticipated at the outset. The vast majority of autonomous solutions that were technically correct came from flawed or misinterpreted problem statements, which occasionally required considerable effort to diagnose. Furthermore, the most challenging step for human experts was not verification, but determining if the solutions already existed in the literature. As AI-generated mathematics grows, the community must remain vigilant of ``subconscious plagiarism'', whereby AI  reproduces knowledge of the literature acquired during training, without proper acknowledgment. Note that formal verification cannot help with any of these difficulties. 

While autonomous efforts on the Erd\H{o}s problems have borne some success, they have also spawned misleading hype and downright misinformation, which have then been amplified on social media platforms---to the detriment of the mathematics community. In addition to the Erd\H{o}s problems, there are many other lists of mathematics conjectures that may become the targets of (semi-)autonomous efforts in the future. We urge such efforts to be attentive to the issues raised here. \\

\noindent \textbf{Acknowledgments.} We thank Thomas Bloom, Gabriel Goldberg, Chris Lambie-Hanson, Vjekoslav Kova\v{c}, Daniel Litt, Insuk Seo, Nat Sothanaphan, Terence Tao, and Wouter van Doorn for help, comments, and advice. \\

\section{Problems autonomously solved by AI}
On these problems, {\Aletheia} found the first correct solution (as far as we can tell). We note, however, that the solution to Erdős-652 is an immediate reduction to the literature.

\subsection{Erdős-652}

\begin{problem}{Erdős-652 \cite{Er97e}}
    Let $x_1,\ldots,x_n\in \mathbb{R}^2$ and let $R(x_i)=\#\{ \lvert x_j-x_i\rvert : j\neq i\}$, where the points are ordered such that\[R(x_1)\leq \cdots \leq R(x_n).\]Let $\alpha_k$ be minimal such that, for all large enough $n$, there exists a set of $n$ points with $R(x_k)<\alpha_kn^{1/2}$. Is it true that $\alpha_k\to \infty$ as $k\to \infty$?
\end{problem}

\begin{remark}
\Aletheia's original output has the correct logical argument, but works with certain incorrect constants throughout (incorrectly cited from \cite{PS98}). More precisely: 
\begin{itemize}
    \item The output cites Theorem~2.1 from ``(Pach--Sharir 1992)'', but no such result exists that we could find. 
    \item A similar theorem is used in the paper \cite{PS98}.
    However, the exponents are different: the model's output says \((2/3,2/3)\), but \cite{PS98} gives \((3/5,4/5)\).
    \item The solution below follows the model output, but with the corrected exponents. 
\end{itemize}
The original output also unnecessarily used $\alpha_{k}+\epsilon$ for any $\epsilon>0$ instead of $\alpha_{k}$ to bound $\frac{R(x_{k})}{n^{1/2}}$ and took $\epsilon\to 0$ which could cause dependency issues since $n$ is dependent on $\epsilon$. This minor issue was also simply fixed in the solution below by using $\alpha_{k}$ itself.

During the writing of this report, a solution by GPT-5.2 Pro was announced \cite{ErdosLLMHunter652}, by similar means of a literature result. However, unlike {\Aletheia}, it instead uses a more recent result of Mathialagan \cite[Theorem 3.6]{mathialagan2019bipartitedistinctdistancesplane}.
\end{remark}
 
\begin{solution}{Solution to {Erdős-652}}
We prove that the conclusion is indeed true. More specifically, we show that we have $\alpha_k = \Omega(k^{1/4})$, which implies the desired result. We begin by citing a core result from the literature, which will be crucial for the proof.

\begin{theorem}(Pach--Sharir \cite[Theorem 1.1]{PS98}) \label{thm : Pach-Sharir incidence bound}
Let $C$ be a set of $n$ simple curves in the plane with the property that
\begin{enumerate}[(i)]
    \item for any $k$ points there are at most $s$ curves of $C$ passing through all of them;
    \item any pair of distinct curves from $C$ intersect in at most $s$ points.
\end{enumerate}
For any set $P$ of $m$ points on the plane, the number of incidences between the points of $P$ and the curves of $C$ is bounded above by
\[c(k,s)\bigl(m^{k/(2k-1)}n^{(2k-2)/(2k-1)}+m+n\bigr),\]
where $c(k,s)$ is a positive constant that depends on $k$ and $s$ but not $m$ or $n$.\\
An incidence is a pair $(p,\gamma)$ with $p\in P$ and $\gamma\in C$.
\end{theorem} 
\textbf{1. Construction of set of points and family of circles} \\
Let $k$ be a fixed positive integer. By the definition of $\alpha_k$, there exists an integer $N_0$ such that for all $n \geq N_0$, there exists a set $P_n = \{x_1, \ldots, x_n\} \subset \mathbb{R}^2$ satisfying:
\[ R(x_k) < \alpha_k n^{1/2}. \]
Since the points are ordered by non-decreasing distinct distance counts $R(x_i)$, it follows that for all $i \in \{1, \ldots, k\}$:
\[ R(x_i) \leq R(x_k) < \alpha_k n^{1/2}. \]
We let  $S = \{x_1, \ldots, x_k\}$ be the subset of the first $k$ points of $P_n$. For each $x_i \in S$, let $D_i$ denote the set of distinct distances from $x_i$ to the other points in $P_n$. That is, $D_i = \{ |p - x_i| : p \in P_n \setminus \{x_i\} \}$. We have $|D_i| = R(x_i) < \alpha_{k} n^{1/2}$.

We construct a family of circles $\mathcal{C}$ defined by these distances:
\[ \mathcal{C} = \bigcup_{i=1}^k \{ \Gamma(x_i, r) : r \in D_i \}, \]
where $\Gamma(c, r)$ is the circle centered at $c$ with radius $r$. Since the circles in $\mathcal{C}$ are all distinct we have
\[ |\mathcal{C}| = \sum_{i=1}^k |D_i| < k \alpha_{k} n^{1/2}. \]

\textbf{2. Lower bound for the number of incidences between $P_n$ and $\mathcal{C}$}\\
We now derive a lower bound for the number of incidences between $P_n$ and $\mathcal{C}$. Consider any point $p \in P_n \setminus S$. For every center $x_i \in S$, the distance $r = |p - x_i|$ is in $D_i$. Thus, $p$ lies on the circle $\Gamma(x_i, r) \in \mathcal{C}$.
Since there are $k$ distinct centers in $S$, every point $p \in P_n \setminus S$ is incident to at least $k$ circles in $\mathcal{C}$.
Consequently, the total number of incidences between $P_n$ and $\mathcal{C}$ is bounded from below by $(n-k)k$.\\

\textbf{3. Application of the known upper bound}\\
We now apply Theorem \ref{thm : Pach-Sharir incidence bound} to $P_n$ and $\mathcal{C}$. Note that since all curves in $\mathcal{C}$ are circles, we may apply Theorem \ref{thm : Pach-Sharir incidence bound} with $k=3$ and $s=2$. Together with the lower bound of the number of incidences obtained above, this implies 
\[(n-k)k\leq c(3,2)\left(n^{3/5}|\mathcal{C}|^{4/5}+n+|\mathcal{C}|\right).\]
Combined with $|\mathcal{C}|<k\alpha_{k}n^{1/2}$ from Part 1 we have
\[(n-k)k\leq C\left(n\alpha_{k}^{4/5}k^{4/5}+n+k\alpha_{k}n^{1/2}\right)\]
for some constant $C>0$ independent of $k$ and $n$. Dividing both sides by $n$ and taking the limit $n\to\infty$ we obtain
\[k\leq C(\alpha_{k}^{4/5}k^{4/5}+1)\]
which implies $\alpha_{k}=\Omega(k^{1/4})$, concluding the proof.
\end{solution}

\subsection{Erdős-1051}

\begin{problem}{Erdős-1051 \cite{ErGr80, Er88}}   
Is it true that if $a_1<a_2<\cdots$ is a sequence of integers with\[\liminf a_n^{1/2^n}>1\]then\[\sum_{n=1}^\infty \frac{1}{a_na_{n+1}}\]is irrational?
\end{problem}

\begin{remark}
The original model output below contains a minor error: it incorrectly takes strict inequalities during the proof of Lemma~\ref{lem:R_n_upper_bound}. 
The solution has been formalised in Lean 4 by Barreto, and a significant generalization is the subject of follow-up work \cite{BKKKZ}, for which {\Aletheia} and Gemini Deep Think also made major contributions  to both the formulation and the proof. 
\end{remark}

\begin{solution}{Solution to Erdős-1051}  We will answer the question affirmatively by proving the following:

\begin{theorem}
    Let $(a_n)$ be a strictly increasing sequence of positive integers satisfying $\liminf_{n\to\infty} a_n^{1/2^n}>1$. Then, the series \[S\coloneq\sum_{n=1}^{\infty}\frac{1}{a_n a_{n+1}}\] is irrational.
\end{theorem}
\begin{proof}
    First, we shall reorient the problem. Set $\alpha\coloneq\liminf_{n\to\infty} a_n^{1/2^n}>1$, and choose any $c$ with $1<c<\alpha$. By the definition of $\liminf$, there exists $n_*(c)$ such that $a_n^{1/2^n}\geq c$ for $n\geq n_*(c)$. Hence, $a_n\geq c^{2^n}\to\infty$, so there exists $N\geq n_*(c)$ with $a_n\geq 2$ for all $n\geq N$. Fix such an $N$, and write $$S=\sum_{n=1}^{N-1}\frac{1}{a_n a_{n+1}}+\sum_{n=N}^{\infty}\frac{1}{a_n a_{n+1}}\eqcolon S_0+S_{\mathrm{tail}}.$$ Note that $S_0\in\mathbb{Q}$, so $S\in\mathbb{Q}$ if and only if $S_{\mathrm{tail}}\in\mathbb{Q}$. Now define the shifted sequence $b_n\coloneq a_{N+n-1}$ for $n\geq 1$, so that $2\leq b_1<b_2<\dots$ are integers and \[S_{\mathrm{tail}}=\sum_{n=1}^{\infty}\frac{1}{b_n b_{n+1}}\eqcolon S'.\] Moreover, the growth condition is preserved under the shift since letting $x_m\coloneq a_m^{1/2^m}$, \[b_n^{1/2^n}=a_{N+n-1}^{1/2^n}=\left(a_{N+n-1}^{1/2^{N+n-1}}\right)^{2^{N-1}}=x_{N+n-1}^{2^{N-1}},\] we recall that $t\mapsto t^{2^{N-1}}$ is continuous and strictly increasing on $(0,\infty)$, so it follows that \[\liminf_{n\to\infty}b_n^{1/2^n}=\left(\liminf_{m\to\infty}a_m^{1/2^m}\right)^{2^{N-1}}=\alpha^{2^{N-1}}>1.\] Thus, it suffices to prove that $S'\notin\mathbb{Q}$.

    To this end, assume, for contradiction, that $S'=p/q$ with $p\in\mathbb{Z}$, $q\in\mathbb{Z}^+$. Define the partial products and tails \[P_n\coloneq\prod_{k=1}^{n}b_k\qquad R_n\coloneq\sum_{k=n}^{\infty}\frac{1}{b_k b_{k+1}}\] for $n\geq 1$, with $P_0=1$. Furthermore, define the terms $K_n\coloneq q P_n R_n$ for $n\geq 1$. We have:

    \begin{lemma}\label{lem:R_n_lower_bound}
        For all $n\geq 1$, one has $K_n\in\mathbb{Z}_{\geq 1}$, hence \[R_n\geq\frac{1}{q P_n}.\]
    \end{lemma}
    \begin{proof}
        Using $S'=p/q$, we have \[K_n=qP_n\left(\frac{p}{q}-\sum_{k=1}^{n-1}\frac{1}{b_k b_{k+1}}\right)=pP_n-q\sum_{k=1}^{n-1}\frac{P_n}{b_kb_{k+1}}.\] For $k\leq n-1$, the integer $b_k b_{k+1}$ divides $P_n=\prod_{j=1}^{n}b_j$, hence each $\frac{P_n}{b_k b_{k+1}}\in\mathbb{Z}$, so $K_n\in\mathbb{Z}$. Also, $R_n>0$ as all terms are positive, so $K_n=qP_nR_n>0$, hence $K_n\geq 1$. Dividing by $qP_n$ yields $R_n\geq 1/(qP_n)$, as claimed.
    \end{proof}

    \begin{lemma}\label{lem:R_n_upper_bound}
        For all $n\geq 1$, \[R_{n+1}\leq\frac{1}{b_{n+1}},\qquad\text{and consequently,}\qquad K_{n+1}\leq qP_n.\]
    \end{lemma}
    \begin{proof}
        Since the $b_k$ are strictly increasing integers, $b_{k+1}-b_k\geq 1$, so \[\frac{1}{b_k b_{k+1}}\leq\frac{b_{k+1}-b_k}{b_k b_{k+1}}=\frac{1}{b_k}-\frac{1}{b_{k+1}}.\] Summing from $k=n+1$ to $m$ gives a telescoping sum \[\sum_{k=n+1}^{m}\frac{1}{b_kb_{k+1}}\leq\sum_{k=n+1}^{m}\left(\frac{1}{b_k}-\frac{1}{b_{k+1}}\right)=\frac{1}{b_{n+1}}-\frac{1}{b_{m+1}}\leq\frac{1}{b_{n+1}}.\] Letting $m\to\infty$ gives $R_{n+1}\leq 1/b_{n+1}$. Multiplying by $qP_{n+1}=qP_n b_{n+1}$ yields \[K_{n+1}=qP_{n+1}R_{n+1}\leq qP_n,\] as claimed.
    \end{proof}
    Using $R_n=\frac{1}{b_n b_{n+1}}+R_{n+1}$ and multiplying by $qP_{n+1}$, we obtain the exact identity
    \begin{equation}
        b_{n+1}K_n=qP_{n-1}+K_{n+1}\qquad (n\geq 1).
    \end{equation}
    By \textbf{Lemma \ref{lem:R_n_upper_bound}} and $K_n\geq 1$, it follows that \[b_{n+1}K_n=qP_{n-1}+K_{n+1}\leq qP_{n-1}+qP_n\implies b_{n+1}\leq q(P_{n-1}+P_n).\] Since $b_n\geq 2$ for all $n$, we have $P_n=b_n P_{n-1}\geq 2P_{n-1}$, i.e.~$P_{n-1}\leq\frac{1}{2}P_n$. Hence, \[b_{n+1}\leq q\left(\frac{1}{2}P_n+P_n\right)=\frac{3}{2}q P_n.\] Multiplying by $P_n$ gives \begin{equation}\label{P_n_ineq}
        P_{n+1}=b_{n+1}P_n\leq CP_n^2,\qquad C\coloneq\frac{3}{2}q.
    \end{equation}

    Now define \[u_n\coloneq\frac{\log P_n}{2^n}\qquad(n\geq 1).\] Taking logarithms in (\ref{P_n_ineq}) yields $\log P_{n+1}\leq 2\log P_n+\log C$, so \begin{equation}\label{u_n_ineq}
        u_{n+1}\leq u_n+\frac{\log C}{2^{n+1}}.
    \end{equation} Set \[v_n\coloneq u_n+\frac{\log C}{2^n}.\] Then, (\ref{u_n_ineq}) implies $v_{n+1}\leq v_n$. Also, $P_n\geq 1$ implies $u_n\geq 0$, hence $v_n\geq 0$. Therefore, $(v_n)$ converges, and since $\frac{\log C}{2^n}\to 0$, $(u_n)$ converges as well. So, let \[Y\coloneq\lim_{n\to\infty}u_n,\qquad\Pi\coloneq e^Y.\] Equivalently, \begin{equation}\label{P_n_lim_identity}
        \lim_{n\to\infty}P_n^{1/2^n}=\Pi.
    \end{equation} Next, since $b_n=P_n/P_{n-1}$, \[\frac{\log b_n}{2^n}=\frac{\log P_n}{2^n}-\frac{\log P_{n-1}}{2^n}=u_n-\frac{1}{2}u_{n-1}\xrightarrow[n\to \infty]{}Y-\frac{1}{2}Y=\frac{Y}{2}.\] Hence, the limit \[L\coloneq\lim_{n\to\infty} b_n^{1/2^n}=e^{Y/2}=\sqrt{\Pi}\] exists. Since the limit exists, it equals the $\liminf$, so by the starting hypothesis, $\liminf_{n\to\infty} b_n^{1/2^n}>1$, we have $L>1$. In particular, $\Pi=L^2>1$.

    Now fix any real $D$ with $1<D<L$. Since $b_n^{1/2^n}\to L$, $\exists n_1(D)\in\mathbb{Z}$ such that for $k\geq n_1(D)$, \[b_k^{1/2^k}\geq D\iff b_k\geq D^{2^k}.\] Thus, for $k\geq n_1(D)$, \[b_kb_{k+1}\geq D^{2^k}D^{2^{k+1}}=D^{3\cdot 2^k},\] and therefore, for all $n\geq n_1(D)$, \begin{equation}\label{R_n_sum_ineq}
        R_n=\sum_{k=n}^{\infty}\frac{1}{b_kb_{k+1}}\leq\sum_{k=n}^{\infty}D^{-3\cdot 2^k}.
    \end{equation} Now let $t_k\coloneq D^{-3\cdot 2^k}$, then $t_{k+1}=t_k^2$, and $t_k\to 0^+$. Choose $n_2(D)$ such that $t_{n_2(D)}\leq 1/2$, then for $n\geq n_2(D)$, we have $t_k\leq 1/2$ for all $k\geq n$, meaning \[t_{k+1}=t_k^2\leq\frac{1}{2}t_k\qquad (k\geq n),\] and so inductively, we obtain \[\sum_{k=n}^{\infty}t_k\leq t_n\sum_{j=0}^{\infty}\left(\frac{1}{2}\right)^j=2t_n.\] Combining with (\ref{R_n_sum_ineq}), we obtain for all $n\geq n_0(D)\coloneq\max\{n_1(D),n_2(D)\}$ that $R_n\leq 2D^{-3\cdot 2^n}$. Now, \textbf{Lemma \ref{lem:R_n_lower_bound}} gives $R_n\geq 1/(qP_n)$ for all $n$, hence for $n\geq n_0(D)$, \[\frac{1}{qP_n}\leq R_n\leq 2D^{-3\cdot 2^n}\implies P_n\geq\frac{1}{2q}D^{3\cdot 2^n}.\] Taking $2^{-n}$-th powers gives \[P_n^{1/2^n}\geq(2q)^{-1/2^n}D^3.\] Letting $n\to\infty$ and using (\ref{P_n_lim_identity}) yields $\Pi\geq D^3$. Since $D$ was arbitrary in $(1,L)$, letting $D\to L^-$ gives $L^2=\Pi\geq L^3$. Since $L>0$, dividing by $L^2$ yields $1\geq L$, contradicting our previously established inequality $L>1$.
    
    Hence, $S'\notin\mathbb{Q}$, meaning $S\notin\mathbb{Q}$, as was to be shown.
\end{proof}
\end{solution}

\section{Problems with parts solved by AI} On these problems, there were multiple questions and {\Aletheia} found the first correct solution to one of the questions.
\subsection{Erdős-654}

\begin{problem}{Erdős-654 \cite{Er97e}}
Let $x_1,\ldots,x_n\in \mathbb{R}^2$ with no four points on a circle. Must there exist some $x_i$ with at least $(1-o(1))n$ distinct distances to other $x_i$?
\end{problem}

\begin{remark}
{\Aletheia} answers the question negatively, with an elementary counterexample consisting of points on $x$-and also on $y$-axes. In \cite[p. 530]{Er97e} where this problem is formulated, Erd\H{o}s does acknowledge that it is `perhaps too optimistic'.

It is pointed out on the corresponding problem page in Bloom's database that in \cite{Er87b,ErPa90}, a weaker question was posed under the additional assumption that no three points are on a line. (In fact, {\Aletheia} also highlighted the additional hypothesis, and gave an answer to that case, which however was incorrect, so we omit that part here.) For this reason, we consider our solution to be only addressing a `part' of the question, even though it fully addresses the formulation on \href{https://www.erdosproblems.com/}{\ttfamily ErdosProblems.com} at the time of this writing.
\end{remark}

\begin{solution}{Solution to {Erdős-654}}
Let $S$ be a set of $n$ points in the Euclidean plane $\mathbb{R}^2$. For each point $x \in S$, we denote by $|\mathcal{D}(x)|$ the number of distinct distances from $x$ to $S \setminus \{x\}$.

We prove that the answer to the question is negative. We construct a family of point sets where every point determines at most roughly $\frac{3}{4}n$ distinct distances. Let $m \ge 10$ be an integer, and let $n = 4m$. Define the index set $K = \{10, 11, \dots, m+9\}$.
We define two sets of points $P$ and $Q$ located on the $y$-axis and $x$-axis, respectively:
\begin{align*}
P &= \left\{ (0, y) \in \mathbb{R}^2 \;\middle|\; y \in \{ 3^k, -3^k : k \in K \} \right\},
\\
Q &= \left\{ (x, 0) \in \mathbb{R}^2 \;\middle|\; x \in \{ 2^j, -2^j : j \in K \} \right\}.
\end{align*}
Let $S = P \cup Q$.
Since the $y$-coordinates of $P$ and the $x$-coordinates of $Q$ are non-zero (powers of primes are positive), neither set contains the origin. Thus $P \cap Q = \emptyset$.
The size of each subset is $2|K| = 2m$. The total size of $S$ is $4m = n$.

\begin{lemma}
No four points of $S$ lie on a circle. \end{lemma}
\begin{proof}
Indeed, a Euclidean circle intersects a straight line in at most two points. Since $P$ lies on the line $x=0$ and $Q$ lies on the line $y=0$, any circle can contain at most two points from $P$ and at most two points from $Q$.
Consequently, for a circle to contain four points of $S$, it must contain exactly two points from $P$ and exactly two points from $Q$.
Suppose such a circle $\mathcal{C}$ exists. Let the points in $P \cap \mathcal{C}$ be $(0, y_1)$ and $(0, y_2)$, and the points in $Q \cap \mathcal{C}$ be $(x_1, 0)$ and $(x_2, 0)$.
The chords formed by these pairs intersect at the origin $(0,0)$. By the Power of a Point Theorem, the product of the signed lengths of the segments from the intersection point must be equal. In terms of coordinates, this implies:
\[
y_1 y_2 = x_1 x_2.
\]
Taking absolute values yields:
\[
|y_1| |y_2| = |x_1| |x_2|.
\]
By construction, $|y_i| = 3^{k_i}$ and $|x_i| = 2^{j_i}$ for some $k_i, j_i \in K$. Substituting these forms:
\[
3^{k_1} \cdot 3^{k_2} = 2^{j_1} \cdot 2^{j_2} \implies 3^{k_1 + k_2} = 2^{j_1 + j_2}.
\]
By the Fundamental Theorem of Arithmetic, a power of 3 equals a power of 2 if and only if both exponents are zero. However, since $k, j \ge 10$, the sums of exponents satisfy $k_1+k_2 \ge 20$ and $j_1+j_2 \ge 20$. Thus, equality is impossible.
This contradiction implies the claim asserted.
\end{proof}

We establish that every point determines fewer than $\frac{3}{4}n$ distinct distances.

\begin{lemma}
For any $u \in S$, distances from $u$ to points on the same axis are integers, while distances to points on the orthogonal axis are irrational. Thus, the set of distances to $P$ and the set of distances to $Q$ are disjoint.
\end{lemma}

\begin{proof}
We consider two cases based on the location of $u$.

\paragraph{Case 1: $u \in P$.}
Let $u = (0, Y)$ with $|Y| = 3^{k_0}$ ($k_0 \ge 10$).
Distances to other points $v \in P$ correspond to distances between $(0, Y)$ and $(0, y)$, which are $|Y-y|$. Since coordinates are integers, these distances are integers.
Distances to points $w = (x, 0) \in Q$ are given by $\sqrt{Y^2 + x^2}$.
Suppose such a distance is an integer $z$. Then $z^2 - Y^2 = x^2$, which implies:
\[
z^2 - 3^{2k_0} = 2^{2j}.
\]
Factoring the difference of squares, $(z-3^{k_0})(z+3^{k_0}) = 2^{2j}$. The factors must be powers of 2, say $2^a$ and $2^b$, with $a < b$. The difference between factors is:
\[(z+3^{k_0}) - (z-3^{k_0}) = 2 \cdot 3^{k_0} = 2^b - 2^a = 2^a(2^{b-a} - 1).\]
Since $3^{k_0}$ is odd, comparing the powers of 2 gives $2^a = 2 \implies a=1$.
The remaining factor yields $3^{k_0} = 2^{b-1} - 1$.
Taking modulo 3: $0 \equiv 2^{b-1} - 1 \pmod 3$, implying $2^{b-1} \equiv 1 \pmod 3$. Thus $b-1$ is even; let $b-1 = 2t$.
Then $3^{k_0} = 2^{2t} - 1 = (2^t - 1)(2^t + 1)$.
Since the product is a power of 3, both factors are powers of 3. Their difference is $(2^t+1) - (2^t-1) = 2$.
The only powers of 3 differing by 2 are 1 and 3. Thus $2^t-1 = 1 \implies t=1$.
Substituting $t=1$ gives $3^{k_0} = 3 \implies k_0 = 1$.
However, by construction $k_0 \ge 10$. Thus, no integer solution exists.

\paragraph{Case 2: $u \in Q$.}
Let $u = (X, 0)$ with $|X| = 2^{j_0}$ ($j_0 \ge 10$).
Distances to points in $Q$ are integers $|X-x|$.
Distances to points in $P$ are $\sqrt{X^2 + y^2}$.
Assume an integer distance $z$. Then $z^2 - 2^{2j_0} = 3^{2k}$.
Factorizing: $(z-2^{j_0})(z+2^{j_0}) = 3^{2k}$.
Let factors be $3^a$ and $3^b$ ($a < b$). Difference:
\[
2 \cdot 2^{j_0} = 3^b - 3^a = 3^a(3^{b-a}-1).
\]
Since 2 is coprime to 3, $a=0$.
Then $2^{j_0+1} = 3^b - 1$.
Consider modulo 4: $0 \equiv 3^b - 1 \implies 3^b \equiv 1 \pmod 4$. Thus $b$ is even; let $b=2s$.
Then $2^{j_0+1} = (3^s-1)(3^s+1)$.
The factors are powers of 2 differing by 2.
Thus $3^s-1 = 2 \implies s=1$.
This implies $2^{j_0+1} = 3^2-1 = 8$, so $j_0+1=3 \implies j_0=2$.
But $j_0 \ge 10$. Thus, no integer solution exists.

In both cases, distances to the orthogonal axis are purely irrational, while distances along the resident axis are integers. Thus, the sets are disjoint.
\end{proof}

\begin{theorem}
For every point $u \in S$, $|\mathcal{D}(u)| < \frac{3}{4}n$.
\end{theorem}

\begin{proof}
Assume without loss of generality that $u \in P$.
By the Lemma, $|\mathcal{D}(u)| = |\mathcal{D}_P| + |\mathcal{D}_Q|$.
\begin{enumerate}
    \item \textbf{Bound for $\mathcal{D}_Q$:}
    The distances are $\{\sqrt{Y^2 + 4^j} \mid j \in K\}$. Since the function $f(j) = Y^2 + 4^j$ is strictly monotonic, there are exactly $|K| = m$ distinct values.
    \[ |\mathcal{D}_Q| = m. \]
    \item \textbf{Bound for $\mathcal{D}_P$:}
    Since $u \in P$, there are $|P|-1 = 2m-1$ other points in $P$. The number of distinct distances cannot exceed the number of target points.
    \[ |\mathcal{D}_P| \le 2m-1. \]
\end{enumerate}
Combining these:
\[
|\mathcal{D}(u)| \le m + (2m-1) = 3m - 1.
\]
Recalling $n = 4m$, we have $3m = \frac{3}{4}n$. Thus:
\[
|\mathcal{D}(u)| \le \frac{3}{4}n - 1 < \frac{3}{4}n.
\]
By symmetry, the same bound applies if $u \in Q$.
\end{proof}

Since $\frac{3}{4}n < (1-\epsilon)n$ for any $\epsilon \in (0, 0.25)$, our construction demonstrates that a set with no four points on a circle need not contain any point with $(1-o(1))n$ distinct distances.
\end{solution}

\subsection{Erdős-1040}

\begin{problem}{Erdős-1040 \cite{EHP58}}
    Let $F\subseteq \mathbb{C}$ be a closed infinite set, and let $\mu(F)$ be the infimum of\[\lvert \{ z: \lvert f(z)\rvert < 1\}\rvert,\]as $f$ ranges over all polynomials of the shape $\prod (z-z_i)$ with $z_i\in F$. Is $\mu(F)$ determined by the transfinite diameter of $F$? In particular, is $\mu(F)=0$ whenever the transfinite diameter of $F$ is $\geq 1$?
\end{problem}

\begin{remark} This problem consists of two questions. {\Aletheia} correctly answered the first question with a simple example, which we were not able to locate in the literature. The argument given by {\Aletheia} for the second part was an incorrect reduction to the literature, specifically to a partial result \cite{krishnapur2025areapolynomiallemniscates} resolving the question under additional assumptions on $F$. We have omitted this second part and lightly edited the first part for flow.
\end{remark}

\begin{solution}{Solution to Erdős-1040}

We prove that $\mu(F)$ is not determined by $d_\infty(F)$. To do so, we exhibit two distinct closed infinite sets, $F_1$ and $F_2$, which have identical transfinite diameters ($d_\infty = 0$) but yield markedly different values for $\mu(F)$.

Recall that for any countable compact set $K$, the transfinite diameter (logarithmic capacity) is zero.\footnote{Human annotation: This claim was not justified or cited by the model output. However, it essentially follows from \cite[Corollary 3.2.5]{Ran95} and the equivalence of transfinite diameter and logarithmic capacity. Moreover, it is easy enough to verify that both sets constructed have transfinite diameter zero.}

\medskip
\textbf{1. Construction of $F_1$ (Positive Area)}

Let $F_1 = \{0\} \cup \left\{ \frac{1}{n} : n \in \mathbb{Z}, n \ge 1 \right\}$.
Since $F_1$ is a countable compact subset of the interval $[0, 1]$, $d_\infty(F_1) = 0$.
Let $P(z) = \prod_{j=1}^m (z - x_j)$ be an arbitrary monic polynomial with roots $x_j \in F_1$.
Since $F_1 \subset [0, 1]$, every root satisfies $x_j \in [0, 1]$.
Consider the disk $D = \{ z : |z - 1/2| < 1/2 \}$.
Since $x_j \in [0, 1]$, we have $|x_j - 1/2| \le 1/2$.
Applying the triangle inequality:
\[ |z - x_j| \le |z - 1/2| + |1/2 - x_j| < \frac{1}{2} + \frac{1}{2} = 1\]
for all $z\in D$. Thus, $|P(z)| = \prod_{j=1}^{m} |z - x_j| < 1^m = 1$ for all $z \in D$.
Consequently, the region $\{ z: \lvert P(z)\rvert < 1\}$ contains $D$, so:
\[ \text{Area}(\{z : |P(z)| < 1\}) \ge \text{Area}(D) = \frac{\pi}{4}. \]
Taking the infimum over all $P$, we have $\mu(F_1) \ge \pi/4$.

\medskip
\textbf{2. Construction of $F_2$ (Vanishing Area)}

Let $R > 4$ be a real parameter. Define
\[ F_2 = \{0, R\} \cup \left\{ \frac{1}{n} : n \ge 1 \right\} \cup \left\{ R + \frac{1}{n} : n \ge 1 \right\}. \]
$F_2$ is a countable compact set, so $d_\infty(F_2) = 0$.
Consider the monic quadratic polynomial $Q(z) = z(z-R)$. Its roots $0, R$ lie in $F_2$.
We compute the area of the lemniscate $L = \{z : |z(z-R)| < 1\}$.
Let $w = z(z-R)$. By completing the square, $w = (z - R/2)^2 - R^2/4$.
Substitute $u = (z - R/2)^2$. The condition $|w| < 1$ transforms to $|u - R^2/4| < 1$.
Solving for $z$: $z = R/2 \pm \sqrt{u}$.
The Jacobian is $|dz/du|^2 = |1/(2\sqrt{u})|^2 = 1/(4|u|)$.
Summing the contributions from both branches of the square root:
\[ \text{Area}(L) = \int_{|u - R^2/4| < 1} \frac{1}{2|u|} dA_u, \]
where $A_u$ denotes the standard area measure on the complex plane.
In the domain of integration we have $|u| \ge R^2/4 - 1$. Thus, the integrand is bounded by $\frac{1}{2(R^2/4 - 1)}$. 
We conclude
\[ \text{Area}(L) \le \pi \cdot \frac{1}{2(R^2/4 - 1)} = \frac{2\pi}{R^2 - 4}. \]
As we increase $R$, this area tends to zero. By choosing $R$ sufficiently large, we ensure $\mu(F_2) < \pi/4$.

\medskip
\textbf{3. Conclusion}

Since $F_1$ and $F_2$ both possess transfinite diameter zero, yet $\mu(F_1) \ge \pi/4$ while $\mu(F_2)$ can be arbitrarily close to zero, $\mu(F)$ is not determined by $d_\infty(F)$.
\end{solution}

\section{Independent rediscovery} On these problems, {\Aletheia} found a correct solution, but human auditors subsequently found an independent solution already in the literature. While \textbf{all credit should certainly be entirely due to the original authors}, we distinguish these solutions from pure literature identification (\S \ref{sec:literature-identification}) because they seem to have stronger implications for {\Aletheia}'s capabilities. However, we reiterate that although we scanned the logs of \Aletheia's reasoning trace to ensure that the solution was not pulled \emph{directly} from the literature solution, there could have been leakage from the pretraining and post-training phases that we would not be able to detect. This risk of ``AI plagiarism'' presents significant concerns for the usage of AI in academic research more broadly.

\subsection{Erdős-397}

\begin{problem}{Erdős-397 \cite{ErGr80}}
    Are there only finitely many solutions to\[\prod_i \binom{2m_i}{m_i}=\prod_j \binom{2n_j}{n_j}\]with the $m_i,n_j$ distinct?
\end{problem}

\begin{remark}\label{rem:397}
    The following solution has been edited from the raw model output mainly for length. The original model output contained a very detailed computational verification of the relevant identity, which we have shortened. All mathematical content present in the following solution was present in the raw model output. After we originally generated this solution, it was found independently using a combination of GPT-5.2 Pro and Aristotle: see \cite{tao2026erdosai}. Later, it was found to be essentially identical to a question used in a Team Selection Test for the Chinese IMO team; see \cite{AoPS12}\footnote{Amusingly, one poster opines there that the problem is ``too easy''.}.
\end{remark}

\begin{solution}{Solution to Erdős-397}
The answer is no; there are infinitely many solutions.

\begin{theorem}
There exist infinitely many distinct pairs of disjoint finite sets of positive integers $(A, B)$ such that
\begin{equation}\label{eq:bin-coef-prod}
\prod_{m \in A} \binom{2m}{m} = \prod_{n \in B} \binom{2n}{n}.
\end{equation}
\end{theorem}
\begin{proof}
Let $C_k = \binom{2k}{k}$ denote the $k$-th central binomial coefficient. To prove the existence of infinitely many solutions, we will explicitly construct an infinite family of disjoint sets satisfying the required product identity.

For every integer $k \ge 3$, define the sets $A_k$ and $B_k$ as follows:
\begin{align*}
    A_k &= \{k, \, 2k-2, \, 8k^2-8k+2\}, \\
    B_k &= \{k-1, \, 2k, \, 8k^2-8k+1\}.
\end{align*}
For any $k\geq 3$, the sets $A_k$ and $B_k$ are disjoint, and each consists of three strictly positive integers. Moreover, each $k$ generates a unique solution pair. Thus, it suffices to prove that \eqref{eq:bin-coef-prod} holds.

Recall the ratio formula for consecutive central binomial coefficients:
\[ \frac{C_x}{C_{x-1}} = \frac{(2x)!}{(x!)^2} \cdot \frac{((x-1)!)^2}{(2x-2)!} = \frac{2x(2x-1)}{x^2} = \frac{2(2x-1)}{x}. \]
Substituting $x=k$, $x=2k$, $x=2k-1$, and $x=8k^2-8k+2$ into the ratio formula, we thus find
\begin{align*}
\frac{\prod_{m\in A_k}C_m}{\prod_{n\in B_k}C_n}
&=\frac{C_k}{C_{k-1}}\cdot\frac{C_{2k-2}}{C_{2k}}\cdot\frac{C_{8k^2-8k+2}}{C_{8k^2-8k+1}}\\
&=\frac{2(2k-1)}k\cdot\frac{2k}{2(4k-1)}\cdot\frac{2k-1}{2(4k-3)}\cdot\frac{2(16k^2-16k+3)}{8k^2-8k+2}\\
&=\frac{2(2k-1)}k\cdot\frac{2k}{2(4k-1)}\cdot\frac{2k-1}{2(4k-3)}\cdot\frac{2(4k-1)(4k-3)}{2(2k-1)^2}\\
&=1,
\end{align*}
as desired.
\end{proof}

\end{solution}

\subsection{Erdős-935}

\begin{problem}{Erdős-935 \cite{Er76d}}
    For any integer $n=\prod p^{k_p}$, let $Q_2(n)$ be the powerful part of $n$, so that\[Q_2(n) = \prod_{\substack{p\\ k_p\geq 2}}p^{k_p}.\]Is it true that, for every $\epsilon>0$ and $\ell\geq 1$, if $n$ is sufficiently large then\[Q_2(n(n+1)\cdots(n+\ell))<n^{2+\epsilon}?\]If $\ell\geq 2$ then is\[\limsup_{n\to \infty}\frac{Q_2(n(n+1)\cdots(n+\ell))}{n^2}\]infinite? If $\ell\geq 2$ then is\[\lim_{n\to \infty}\frac{Q_2(n(n+1)\cdots(n+\ell))}{n^{\ell+1}}=0?\]
\end{problem}
\begin{remark}
The problem consists of three questions, the second of which was solved by {\Aletheia}. While the original solution (for this part) was mathematically correct, it was edited in order to address the following aesthetic details:
\begin{itemize}
    \item The original output of {\Aletheia} elaborates in detail about why it suffices to show that $Q_{2}(n_{k}+2)$ is unbounded, analyzing each term of $n_{k},n_{k}+1$ and $n_{k}+2$, while the fact seemed straightforward after the construction $n_{k}=8y_{k}^2$.
    \item The notation for equality in $\mathbb F_{p^2}$ was changed from $=$ to $\equiv$, since elements from $\bar{\mathbb Q}$ and $\mathbb F_{p^2}$ are frequently referred to by the same notation. 
\end{itemize}
The argument is certainly well-known to experts, and is at the level of a graduate student exercise. The observation that there are infinitely many consecutive squarefull numbers $n,n+1$ via solutions to the Pell equation is standard, see for example \cite{reuss2014pairskfreenumbersconsecutive} for a more recent quantitative version of this, as well as Mahler's observation listed on the entry for Erdős-364 \cite{BloomErdos364}. From there, one only needs to gain the square of an increasingly large prime in $n+2$, which also follows from basic properties of the Pell equation. 

We also remark that it is well-known that the ABC Conjecture implies that for each $\ell$ there are only finitely many choices of $n$ such that the list $n,\,\ldots,\,n+\ell$ contains $3$ squarefull numbers. Reasoning along similar lines, one can in fact show that the ABC Conjecture implies that the third question in the problem statement has a positive answer.

\end{remark}

\begin{solution}{Solution to the second question in Erdős-935}
We start by noting that for $\ell\geq 2$ the product $P(n, 2) = n(n+1)(n+2)$ is a divisor of $P(n, \ell)$. By the property that $A \mid B \implies Q_2(A) \le Q_2(B)$, we have
\[ \frac{Q_2(P(n, \ell))}{n^2} \ge \frac{Q_2(P(n, 2))}{n^2}. \]
and thus proving the divergence for $\ell=2$ is sufficient. We construct a specific sequence of integers $\{n_k\}$ along which the normalized powerful part is unbounded.

Consider the Pell equation
\[
x^2 - 8y^2 = 1.
\]
Defining the sequences $(x_k)_{k \ge 1}$ and $(y_k)_{k \ge 1}$ of positive integers by the relation
\[ x_k + y_k\sqrt{8} = \alpha^k\]
for $\alpha=3+\sqrt{8}$, we have $x_{k}^{2}-8y_{k}^{2}=1$.
Let $n_{k}=8y_{k}^{2}$. Since $n_k = 8y_k^2$ is a powerful number and $n_{k}+1=x_{k}^2$ is a perfect square we have $Q_{2}(n_{k})=n_{k}$ and $Q_{2}(n_{k}+1)=n_{k}+1$. Since $Q_{2}(ab)\geq Q_{2}(a)Q_{2}(b)$ for any positive integers $a,b$ we have 
\begin{align*}Q_{2}(n_{k}(n_{k}+1)(n_{k}+2))&\geq Q_{2}(n_{k})Q_{2}(n_{k}+1)Q_{2}(n_{k}+2)\\&=n_{k}(n_{k}+1)Q_{2}(n_{k}+2)
\end{align*}
and therefore
\[\frac{Q_{2}(n_{k}(n_{k}+1)(n_{k}+2))}{n_{k}^2}>Q_{2}(n_{k}+2).\]
It remains to show that $Q_{2}(n_{k}+2)$ is unbounded.

\begin{lemma} \label{lemma : n_k+2 modulo p^2}
For any prime $p$ with $p \equiv 5 \pmod 8$ there exists a positive integer $k$ such that $n_k + 2 \equiv 0 \pmod{p^2}$.
\end{lemma}
\begin{proof}
Recall $n_k+2 = x_k^2+1$. We seek $k$ such that $x_k^2 \equiv -1 \pmod{p^2}$.
Since $p \equiv 5 \pmod 8$, we have the Legendre symbol $(2/p) = -1$. Thus $p$ is inert in $\mathbb{Z}[\sqrt{2}]$.
Note $\alpha = 3 + \sqrt{8} = (1+\sqrt{2})^2$.
In $\mathbb{F}_{p^2} \cong \mathbb{Z}[\sqrt{2}]/(p)$, the Frobenius map $\sigma(z) = z^p$ fixes elements of $\mathbb{F}_p$ and sends $\sqrt{2} \to -\sqrt{2}$.
Therefore, writing $\equiv$ for equality in $\mathbb F_{p^2}$,
\[ (1+\sqrt{2})^p \equiv 1 - \sqrt{2}. \]
Then
\[ \alpha^{(p+1)/2} \equiv \left((1+\sqrt{2})^2\right)^{(p+1)/2} \equiv (1+\sqrt{2})^{p+1} \equiv (1+\sqrt{2})(1-\sqrt{2}) \equiv 1 - 2 \equiv -1. \]
Let $m = (p+1)/2$. Note that $m$ is an odd integer since $p \equiv 5 \pmod 8$. We write $\alpha^m = -1 + p\delta$ for some $\delta \in \mathbb{Z}[\sqrt{2}]$.
Raising to the power $p$ (which is odd):
\[ \alpha^{mp} = (-1 + p\delta)^p = \sum_{j=0}^p \binom{p}{j}(-1)^{p-j}(p\delta)^j. \]
Since $p \ge 3$, terms with $j \ge 1$ are divisible by $p^2$. The term for $j=0$ is $(-1)^p = -1$.
Thus, $\alpha^{mp} \equiv -1 \pmod{p^2}$.
Let $M = mp$. Since $m$ and $p$ are odd, $M$ is odd.
Define $k$ by $2k = M+1$. Note that $k$ is an integer.
Then
\[ \alpha^{2k} = \alpha^{M+1} = \alpha^M \cdot \alpha \equiv (-1) \cdot \alpha = -\alpha \pmod{p^2}. \]
Since $\alpha$ is a unit, $\alpha^{-2k} \equiv -\alpha^{-1} \pmod{p^2}$.
From the relation $2x_k = \alpha^k + \alpha^{-k}$, squaring yields:
\[ 4x_k^2 = \alpha^{2k} + \alpha^{-2k} + 2. \]
Substituting the modular expressions:
\[ 4x_k^2 \equiv -\alpha - \alpha^{-1} + 2 \pmod{p^2}. \]
Since $\alpha + \alpha^{-1} = (3+\sqrt{8}) + (3-\sqrt{8}) = 6$, we have:
\[ 4x_k^2 \equiv -6 + 2 = -4 \pmod{p^2}. \]
Since $p$ is odd, 4 is a unit modulo $p^2$. Dividing by 4 gives:
\[ x_k^2 \equiv -1 \pmod{p^2}. \]
Thus $n_k+2 = x_k^2+1 \equiv 0 \pmod{p^2}$.
\end{proof}

By Dirichlet's theorem on primes in arithmetic progressions, there exist infinitely many primes $p \equiv 5 \pmod 8$. For each such $p$ we may apply Lemma \ref{lemma : n_k+2 modulo p^2} to find an $n_{k}$ with $n_{k}+2\equiv 0\pmod {p^2}$, implying $Q_2(n_{k}+2)\geq p^2$. This means $Q_{2}(n_{k}+2)$ is unbounded, concluding the proof.

\end{solution}

\begin{addendum}[Feb 5, 2026] After the posting of this paper, Wouter van Doorn observed that the question solved is almost identical to a question in Erd\H{o}s-367. Moreover, the construction in {\Aletheia}'s proof is the same as that in van Doorn's Nov 20, 2025 comment on the Erd\H{o}s-367 problem page. We checked the thinking logs for Erd\H{o}s-935's solution attempt and confirmed that it did not access said problem page; also, the comment occurred after the base model's knowledge cutoff date, so that it was not in the training data. In light of this information, we re-classified the solution as an Independent Rediscovery. 
\end{addendum}

\subsection{Erdős-659}

\begin{problem}{Erdős-659 \cite{Er97e}}
    Is there a set of $n$ points in $\mathbb{R}^2$ such that every subset of $4$ points determines at least $3$ distances, yet the total number of distinct distances is\[\ll \frac{n}{\sqrt{\log n}}?\]
\end{problem}

\begin{remark}
{\Aletheia} answered this question affirmatively. Later, while checking the literature, we found essentially the same result in a 2014 blog post of Adam Sheffer \cite{sheffer2014point}. However, the argument there (due to Sheffer and Ben Lund) does not attempt to resolve this problem in full, as it does not discuss Configuration 2 below. During the preparation of this paper (but after all of our solutions were generated and evaluated), a full solution based on a Gemini 3.0 response was written up by Grayzel \cite{grayzel2026solutionproblemerdhosconcerning}. By searching through {\Aletheia}'s thinking logs, we verified at least that Sheffer's blog post was not accessed during the model's output generation.
\end{remark}

\begin{solution}{Solution to {Erdős-659}}
The answer to the question is in the affirmative, and we prove so by constructing the set $P$ as a subset of a specific integer lattice in $\mathbb{R}^2$. This lattice is chosen to ensure that all squared distances are integers (ruling out irrational distance ratios) and that it possesses no geometric symmetries that would allow for a 4-point subset with fewer than 3 distances.

\medskip
\textbf{1. Construction of the Lattice and Point Set}

Let $K = \mathbb{Q}(\sqrt{-7})$. Consider the ring of integers $\mathcal{O}_K$ generated over $\mathbb{Z}$ by $1$ and $\omega = \frac{1 + i\sqrt{7}}{2}$. We embed $\mathcal{O}_K$ into $\mathbb{R}^2$ to form the lattice $\Lambda$:
\[ \Lambda = \left\{ u = m(1, 0) + k\left(\frac{1}{2}, \frac{\sqrt{7}}{2}\right) \;\middle|\; m, k \in \mathbb{Z} \right\}. \]
The squared Euclidean norm of a vector $u = m + k\omega$ is given by the field norm $N(u)$:
\[ \|u\|^2 = \left(m + \frac{k}{2}\right)^2 + \left(\frac{k\sqrt{7}}{2}\right)^2 = m^2 + mk + 2k^2. \]
Let $Q(m, k) = m^2 + mk + 2k^2$. Since $m, k \in \mathbb{Z}$, the squared distance between any two points in $\Lambda$ is a non-negative integer. Thus, the ratio of any two squared distances in $\Lambda$ must be rational.\\
Let $P_n$ be the subset of $\Lambda$ consisting of the $n$ points closest to the origin. Since there are at least $\Omega(n)$ points closer to the origin than $u=\lceil\sqrt{n}\rceil(1,0)+\lceil\sqrt{n}\rceil\left(\frac{1}{2},\frac{\sqrt{7}}{2}\right)$, the points of $P_n$ are contained within a disk of radius $O(R)$ for $R=\|\lceil\sqrt{n}\rceil(1,0)+\lceil\sqrt{n}\rceil\left(\frac{1}{2},\frac{\sqrt{7}}{2}\right)\|=2\lceil\sqrt{n}\rceil$. We note $R^2 = \Theta(n)$.

\medskip
\textbf{2. Upper bound on the number of distinct Distances in $P_n$}

The set of distinct distances determined by $P_n$ is a subset of $\{\sqrt{N} \mid N \in \mathbb{Z}_{\ge 0}, N \le 4R^2, N \text{ represented by } Q\}$. 
The quadratic form $Q(m, k) = m^2 + mk + 2k^2$ is positive definite with discriminant $\Delta = 1^2 - 4(2) = -7$.
A theorem by Bernays \cite{Bernays1912}, generalizing the Landau--Ramanujan theorem, states that the number of positive integers $h \le X$ representable by a positive definite binary quadratic form is asymptotically
\[ \mathcal{B}(X) \sim C \frac{X}{\sqrt{\log X}}, \]
for some constant $C > 0$. Since the maximum squared distance in $P_n$ is at most $X = 4R^2 = O(n)$, the number of distinct distances satisfies:
\[ |D(P_n)| \le \mathcal{B}(O(n)) = O\left(\frac{n}{\sqrt{\log n}}\right). \]
Thus we have verified that our constructed $P_n$ indeed satisfies the condition about the number of distinct distances.

\medskip
\textbf{3. Lower bound on the number of distances determined by 4 points}

We conclude by proving that every subset of 4 distinct points in $\Lambda$ determines at least 3 distinct distances. Since no set of 4 points in $\mathbb{R}^2$ can determine exactly 1 distance (as the regular tetrahedron is not planar), it suffices to prove that $\Lambda$ contains no subset of 4 points determining exactly 2 distinct distances.\\
One can deduce from elementary observations that any such set must be similar to one of the following configurations:
\begin{enumerate}
    \item  A square.
    \item An isosceles trapezoid with its base length equal to its diagonal length.
    \item  A rhombus composed of two equilateral triangles sharing an edge.
    \item The equilateral triangle with its centroid.
    \item The equilateral triangle along with another point that has equal distance to two of the vertices of the triangle, in which case the convex hull of the four points forms an isosceles triangle.
    \item The equilateral triangle along with another point that has equal distance to two of the vertices of the triangle, in which case the convex hull of the four points forms a kite.
\end{enumerate}
Configuration 2 involves an irrational squared distance ratio of $4\cos^{2}{\frac{2\pi}{5}}=\frac{3-\sqrt{5}}{2}$ and therefore cannot exist in $\Lambda$. Since configurations 3 through 6 all contain an equilateral triangle as a sub-configuration, it remains to show that $\Lambda$ contains no squares or equilateral triangles.

\smallskip
\textbf{3.1. The lattice $\Lambda$ contains no equilateral triangles.}

Suppose that points $0, u, v \in \Lambda$ form an equilateral triangle. Representing $u,v$ by complex numbers corresponding to the lattice points we may let $v = e^{\pm i\pi/3} u$, which means $v/u = \frac{1 \pm i\sqrt{3}}{2}$.
For $u, v \in \Lambda$, the quotient $v/u$ must lie in the field of fractions of $\mathcal{O}_K$, which is $\mathbb{Q}(\sqrt{-7})$.
However, $\frac{1 \pm i\sqrt{3}}{2} \notin \mathbb{Q}(\sqrt{-7})$ because $\sqrt{-3} \notin \mathbb{Q}(\sqrt{-7})$ (since $3/7$ is not a square in $\mathbb{Q}$), a contradiction.

\smallskip
\textbf{3.2. The lattice $\Lambda$ contains no squares.}

Suppose points $0, u, w, u+w \in \Lambda$ form a square. Considering the complex number representations we may let $w = \pm i u$, i.e. $w/u = \pm i$. However we do not have $i\in\mathbb{Q}(\sqrt{-7})$, and analogously to 3.1 we reach a contradiction.

\medskip
From the conclusions from Part 2 and 3 we conclude that our constructed set $P_n$ satisfies the desired conditions.
    
\end{solution}

\subsection{Erdős-1089}

\begin{problem}{Erdős-1089 \cite{Er75f}}
    Let $g_d(n)$ be minimal such that every collection of $g_d(n)$ points in $\mathbb{R}^d$ determines at least $n$ many distinct distances. Estimate $g_d(n)$. In particular, does\[\lim_{d\to \infty}\frac{g_d(n)}{d^{n-1}}\]exist?
\end{problem}

\begin{remark}
Our human experts eventually discovered that this problem was solved by Bannai--Bannai in \cite[Remark 3(ii)]{Bannai81}. The authors seem unaware that their remark addresses a question of Erdős. 
Strictly speaking, {\Aletheia}'s thinking trace never accesses \cite{Bannai81}, although it mentions the sequel paper \cite{BBS83} many times, whose bibliographic data is extremely similar. Although this qualifies the solution as ``Independent Rediscovery'' according to our definition, it seems like a very likely case for ``subconscious'' AI plagiarism. 

We modified the original model's output only in the part relative to the upper bound. The model tried to give a sketchy overview of the proof of the upper bound from the work of Bannai--Bannai--Stanton (which it explicitly cited). In the interest of brevity we directly cite their main theorem. 
\end{remark}

\begin{solution}{Solution to Erdős-1089}
The problem asks for an estimate of $g_d(n)$ and the determination of the limit $\lim_{d\to \infty}\frac{g_d(n)}{d^{n-1}}$.

\begin{theorem}
Let $n \ge 1$ be an integer. Let $g_d(n)$ denote the minimal integer such that every set of $g_d(n)$ distinct points in $\mathbb{R}^d$ determines at least $n$ distinct non-zero distances.
\begin{itemize}
    \item For $n=1$, $g_d(1) = 2$, and the limit is $2$.
    \item For $n \ge 2$, the limit exists and is equal to $\frac{1}{(n-1)!}$.
\end{itemize}
\end{theorem}

\begin{proof} 

\textbf{1. Reduction to $s$-distance sets}

Let $P \subset \mathbb{R}^d$ be a finite set of points. The set of distinct non-zero distances determined by $P$ is defined as $D(P) = \{ \|x-y\| : x,y \in P, x \neq y \}$.
The definition of $g_d(n)$ implies that if $|P| \ge g_d(n)$, then $|D(P)| \ge n$.
Taking the contrapositive, if $|D(P)| \le n-1$, then $|P| < g_d(n)$.
Thus, $g_d(n) - 1$ is the maximum possible cardinality of a set $P \subset \mathbb{R}^d$ that determines at most $n-1$ distinct distances.
Let $s = n-1$. Let $M_d(s)$ denote the maximum size of an $s$-distance set in $\mathbb{R}^d$ (a set determining at most $s$ distinct distances).
Then:
\[ g_d(n) = M_d(s) + 1. \]

\textbf{2. Case $n=1$}

If $n=1$, then $s=0$.
A set with 0 distinct non-zero distances cannot contain any pair of distinct points. Thus, it contains at most 1 point.
So $M_d(0) = 1$.
Therefore, $g_d(1) = 1 + 1 = 2$.
The limit is:
\[ \lim_{d\to \infty} \frac{g_d(1)}{d^{1-1}} = \frac{2}{1} = 2. \]

\textbf{3. Case $n \ge 2$}

Here $s = n-1 \ge 1$. We estimate $M_d(s)$.

\textit{Upper Bound.}
We use the bound for $s$-distance sets in Euclidean space established by Bannai, Bannai, and Stanton \cite{BBS83}. Theorem $1$ therein states precisely the upper bound
$$M_d(s) \leq \binom{d+s}{s}.
$$
\textit{Lower Bound.}
We construct a set with a large number of points determining at most $s$ distances.
Consider the vector space $\mathbb{R}^{d+1}$. Let $V$ be the set of binary vectors with exactly $s$ ones (Hamming weight $s$):
\[ V = \left\{ v \in \{0,1\}^{d+1} : \sum_{i=1}^{d+1} v_i = s \right\}. \]
All points in $V$ satisfy the equation $\sum x_i = s$, defining a hyperplane $H \subset \mathbb{R}^{d+1}$. Since $H$ is a $d$-dimensional affine subspace, it is isometric to $\mathbb{R}^d$. Thus, $V$ is isometrically embeddable in $\mathbb{R}^d$.
The cardinality of $V$ is:
\[ |V| = \binom{d+1}{s}. \]
Now, let $u, v \in V$ be distinct points. The squared Euclidean distance is:
\[ \|u-v\|^2 = \sum_{i=1}^{d+1} (u_i - v_i)^2. \]
Since $u_i, v_i \in \{0,1\}$, $(u_i - v_i)^2 = |u_i - v_i|$.
Also, $\sum |u_i - v_i| = \text{weight}(u) + \text{weight}(v) - 2|u \cap v| = s + s - 2k = 2(s-k)$, where $k$ is the number of common ones.
Since $u \ne v$, the intersection size $k$ is an integer satisfying $0 \le k \le s-1$.
Thus, the possible squared distances are $\{ 2(s-k) : k = 0, 1, \dots, s-1 \}$.
There are at most $s$ distinct values (specifically $\sqrt{2}, \sqrt{4}, \dots, \sqrt{2s}$).
Therefore, $V$ is an $s$-distance set.
This implies:
\[ M_d(s) \ge |V| = \binom{d+1}{s}. \]
Substituting $s=n-1$:
\[ g_d(n) \ge \binom{d+1}{n-1} + 1. \]

\textbf{4. Calculating the Limit}

We combine the inequalities for $n \ge 2$:
\[ \binom{d+1}{n-1} + 1 \le g_d(n) \le \binom{d+n-1}{n-1} + 1. \]
Dividing by $d^{n-1}$:
\[ \frac{\binom{d+1}{n-1}}{d^{n-1}} + \frac{1}{d^{n-1}} \le \frac{g_d(n)}{d^{n-1}} \le \frac{\binom{d+n-1}{n-1}}{d^{n-1}} + \frac{1}{d^{n-1}}. \]
We evaluate the limits of the bounding terms as $d \to \infty$.
Using the expansion $\binom{N}{k} = \frac{N^k}{k!} + O(N^{k-1})$, we have:
\[ \lim_{d\to \infty} \frac{\binom{d+1}{n-1}}{d^{n-1}} = \lim_{d\to \infty} \frac{(d+1)\cdots(d-n+3)}{d^{n-1}(n-1)!} = \frac{1}{(n-1)!}. \]
Similarly,
\[ \lim_{d\to \infty} \frac{\binom{d+n-1}{n-1}}{d^{n-1}} = \lim_{d\to \infty} \frac{(d+n-1)\cdots(d+1)}{d^{n-1}(n-1)!} = \frac{1}{(n-1)!}. \]
Since the terms $1/d^{n-1}$ vanish as $d \to \infty$, both the upper and lower bounds converge to $\frac{1}{(n-1)!}$.
By the Squeeze Theorem,
\[ \lim_{d\to \infty}\frac{g_d(n)}{d^{n-1}} = \frac{1}{(n-1)!}. \qedhere \]

\end{proof}

\textbf{Answer:}
For $n \ge 2$, the limit exists and equals $\frac{1}{(n-1)!}$.
For $n=1$, the limit is $2$.

\medskip
\textbf{References}
\begin{enumerate}
    \item \cite{BBS83} Bannai, E., Bannai, E., and Stanton, D. (1983). An upper bound for the cardinality of an s-distance subset in real Euclidean space, II. \textit{Combinatorica}, 3(2), 147–152.
\end{enumerate}
\end{solution}

\section{Literature identification}\label{sec:literature-identification} On these problems, {\Aletheia} found that a solution was already explicitly in the literature, despite the problem being marked ``Open'' on Bloom's website.

\subsection{Erdős-333}

\begin{problem}{Erdős-333 \cite{ErGr80}}
    Let $A\subseteq \mathbb{N}$ be a set of density zero. Does there exist a basis $B$ such that $A\subseteq B+B$ and\[\lvert B\cap \{1,\ldots,N\}\rvert =o(N^{1/2})\]for all large $N$?
\end{problem}

\begin{remark}
    This problem briefly made initial waves on social media, as it was thought to have been the first fully AI-generated solution to an Erd\H{o}s problem (due to an honest oversight). Our team was able to quickly correct this misconception on \href{https://www.erdosproblems.com/}{\ttfamily ErdosProblems.com} (response by user ``KoishiChan'' in the entry for Erdős-333), knowing that {\Aletheia} had already discovered the solution in the literature. Somewhat puzzlingly, the solution is an immediate consequence of a Theorem of Erd\H{o}s himself (in joint work with Newman) \cite{EN77}, in a paper that predates the question. The model's answer recapitulated the argument of \cite{EN77}, without explicitly referring to Erd\H{o}s--Newman, but then gave the reference \cite{EN77} in its last line. We replaced this recapitulation with a brief reference to \cite{EN77}.
\end{remark}

\begin{solution}{Solution to {Erdős-333}}
A negative answer to this question follows from \cite[Theorem 2]{EN77}, which is on the minimum size of a set \(B\) such that \(B+B\) contains a given finite set of positive integers, together with the method of extending this result to infinite sets. The latter is not described in \cite{EN77} with detail but referred to as \emph{the familiar process of condensation}.

To elaborate, \cite[Theorem 2]{EN77} shows that for almost all subsets $A$ of $[N]$ with size $n$ which contain \(N\), the minimum size of $B$ with $A \subseteq B + B$ is at least $\min(n/\log N, N^{1/2} / 2)$. Let us say $A$ is \emph{bad with respect to $N$} if this is the case. We choose a sequence of sufficiently rapidly growing dyadic $N_0, N_1, N_2, \dots$. Apply the theorem to obtain, for each sufficiently large \(k\), a set $A_k \subseteq [N_k, 2N_k]$ of size $n_k = \lfloor N_k^{0.6} \rfloor$ which is bad with respect to $2N_k$. This is possible because the proof of the theorem provides a quantitative upper bound on the number of bad sets with respect to \(N\). Then $A \coloneqq \bigcup_k A_k$ has density zero, and there is no set \(B\) such that \(A\subseteq B+B\) and \(|B\cap[N]|=o(N^{1/2})\) for all sufficiently large \(N\).

\medskip
\textbf{References}\begin{enumerate} \item \cite{EN77} P. Erdős and D. J. Newman, ``Bases for sets of integers'', \textit{Journal of Number Theory}, vol.~9, no.~4, pp.~420--425, 1977.
\end{enumerate}
\end{solution}

\subsection{Erdős-591}

\begin{problem}{Erdős-591 \cite{Er82e, Er87}}
    Let $\alpha$ be the infinite ordinal $\omega^{\omega^2}$. Is it true that in any red/blue colouring of the edges of $K_\alpha$ there is either a red $K_\alpha$ or a blue $K_3$?
\end{problem}

\begin{remark} In \cite[Theorem 6.2]{darby99}, it is recorded that Darby and Schipperus (independently, according to~\cite{HajLar10}) solved the problem in the affirmative, and that Larson strengthened the result to $K_4$. Darby's result was not published, and ten years after~\cite{darby99}, Schipperus published a stronger result~\cite{Sch10}. \Aletheia~also noted an extra result on the negative direction; namely, $ \omega^{\omega^2} \not\to (\omega^{\omega^2}, 5)^2 $. This result  appears in~\cite{darby99} (published in~\cite{Larson00}), but \Aletheia~gave an invalid argument for this extra part, which we omit here.
\end{remark}

\begin{solution}{Solution to {Erdős-591}}
\begin{theorem}
    The partition relation $\omega^{\omega^2} \to (\omega^{\omega^2}, 3)^2$ holds. That is, for any coloring of the edges of the complete graph on $\omega^{\omega^2}$ with two colors (red and blue), there exists either a red complete subgraph of order type $\omega^{\omega^2}$ or a blue complete subgraph of size 3.
\end{theorem}

\begin{proof}

We establish the validity of the relation by invoking the characterization of countable partition ordinals derived by Rene Schipperus and Carl Darby.

\medskip
\textbf{1. Representation of the Ordinal}\\
Let $\alpha = \omega^{\omega^2}$. We express $\alpha$ in the standard form used for partition ordinals:
$$ \alpha = \omega^{\omega^\beta} $$
Comparing exponents, we identify $\beta = 2$.

\medskip
\textbf{2. Characterization of Countable Partition Ordinals}\\
The set of countable ordinals $\alpha$ that satisfy the partition relation $\alpha \to (\alpha, 3)^2$ was characterized by Schipperus.

\begin{theorem}[Schipperus, \cite{Sch10}]
    For a countable ordinal $\beta$, the relation $\omega^{\omega^\beta} \to (\omega^{\omega^\beta}, 3)^2$ holds if and only if $\beta$ satisfies one of the following conditions:
    1.  $\beta$ is additively indecomposable.
    2.  $\beta$ is the sum of exactly two additively indecomposable ordinals, i.e., $\beta = \gamma + \delta$ where $\gamma, \delta$ are indecomposable and $\gamma \ge \delta \ge 1$.
\end{theorem}

\textbf{3. Verification for $\beta = 2$}\\
We test the ordinal $\beta = 2$ against the stated conditions.
Recall that an ordinal is additively indecomposable if it is a power of $\omega$ (including $\omega^0 = 1$).
We decompose $\beta$:
$$ 2 = 1 + 1 $$
Here, $\gamma = 1$ and $\delta = 1$.
Since $1 = \omega^0$, both $\gamma$ and $\delta$ are additively indecomposable.
The inequality condition requires $\gamma \ge \delta \ge 1$.
Substituting the values, we have $1 \ge 1 \ge 1$, which is clearly true.

Thus, $\beta = 2$ satisfies Condition 2 of Schipperus's Theorem.

Additionally, Larson extended this result in \cite{Larson00} to show that the relation actually holds for target size 4 (unpublished, but announced in~\cite{darby99}):
$$ \omega^{\omega^2} \to (\omega^{\omega^2}, 4)^2. $$
Since the existence of a blue $K_4$ implies the existence of a blue $K_3$, the positive result for 3 is a direct consequence of this stronger theorem.
\end{proof}

\medskip
\textbf{References}
\begin{enumerate}
    \item\cite{Sch10} Schipperus, R. (2010). \textit{Countable partition ordinals}. {Annals of Pure and Applied Logic}, 161(10), 1195–1215.
    
    
    \item\cite{darby99} Darby, C. (1999). \textit{Negative Partition Relations for Ordinals}. Journal of Combinatorial Theory, Series B, 76, 205--222.

    \item\cite{Larson00} Larson, J. A. (2000). \textit{An ordinal partition avoiding pentagrams}. The Journal of Symbolic Logic. Vol.~65, No.~3, pp.~969--978.
\end{enumerate}
\end{solution}

\subsection{Erdős-705}

\begin{problem}{Erdős-705 \cite{Er81}}

Let $G$ be a finite unit distance graph in $\mathbb{R}^2$ (i.e. the vertices are a finite collection of points in $\mathbb{R}^2$ and there is an edge between two points if and only if the distance between them is $1$). Is there some $k$ such that if $G$ has girth $\geq k$ (i.e. $G$ contains no cycles of length $<k$) then $\chi(G)\leq 3$?
\end{problem}

\begin{remark}
    \Aletheia's response identifies a result of O'Donnell \cite{OD99Thesis} from 2000 which immediately solves this problem. This is another instance where the model spends ``slack (reasoning) moves'' on obvious deductions, for example allocating a sentence to argue that girth $> k$ implies girth $\geq k$, and another sentence to observe that $4 \leq 3$ is a contradiction. We streamlined the proof by trimming these obvious steps. 
\end{remark}

\begin{solution}{Solution to {Erdős-705}}

The answer to the question is \textbf{No}.

\begin{theorem}
    There does not exist an integer $k$ such that every finite unit distance graph $G$ in $\mathbb{R}^2$ with girth at least $k$ satisfies $\chi(G) \leq 3$. 
\end{theorem}

\begin{proof}
Recall that a unit distance graph $G = (V, E)$ is defined by a finite set of vertices $V \subset \mathbb{R}^2$ where $\{u, v\} \in E$ if and only if $\|u - v\| = 1$.

We invoke a result established by Paul O'Donnell in 2000.

\begin{theorem}[O'Donnell, 2000]
For every integer $g \ge 3$, there exists a finite unit distance graph $H_g$ in $\mathbb{R}^2$ such that:
\begin{enumerate}
    \item The girth of $H_g$ is greater than $g$.
    \item The chromatic number of $H_g$ is 4.
\end{enumerate}
\end{theorem}

This immediately completes the proof. 
\end{proof}

\medskip
\textbf{References}
\begin{enumerate}
    \item \cite{OD99Thesis} O'Donnell, Paul. ``High-girth unit-distance graphs''. Rutgers The State University of New Jersey, School of Graduate Studies, 1999.
    
    \item \cite{OD00I}  O'Donnell, Paul., ``Arbitrary Girth, 4-Chromatic Unit Distance Graphs in the Plane I: Graph Description''. \textit{Geombinatorics} 9, no.~3 (2000): 145–152.
    
    \item \cite{OD00II} O'Donnell, Paul., ``Arbitrary Girth, 4-Chromatic Unit Distance Graphs in the Plane II: Graph Embedding''. \textit{Geombinatorics} 9, no.~4 (2000): 180–193.
\end{enumerate}
		
\end{solution}											

\subsection{Erdős-992}

\begin{problem}{Erdős-992 \cite{Er64}}
    Let $x_1<x_2<\cdots$ be an infinite sequence of integers. Is it true that, for almost all $\alpha \in [0,1]$, the discrepancy\[D(N)=\max_{I\subseteq [0,1]} \lvert \#\{ n\leq N : \{ \alpha x_n\}\in I\} - \lvert I\rvert N\rvert\]satisfies\[D(N) \ll N^{1/2}(\log N)^{o(1)}?\]Or even\[D(N)\ll N^{1/2}(\log\log N)^{O(1)}?\]
\end{problem}

\begin{remark} This problem is solved in a paper of Berkes--Philipp \cite{BP94}. {\Aletheia} correctly cites \cite{BP94}, and explains in excess detail why it answers the question negatively. We deleted the explanation as it is already given explicitly in the original paper. 
\end{remark}

\begin{solution}{Solution to Erdős-992}
    The answer is \textbf{no}. There exists a strictly increasing sequence of integers $(x_n)_{n \ge 1}$ such that for almost all $\alpha \in [0,1]$, the discrepancy
    \[ D(N) = \sup_{I\subseteq [0,1]} \lvert \#\{ n\leq N : \{ \alpha x_n\}\in I\} - \lvert I\rvert N\rvert \]
    satisfies
    \[ \limsup_{N \to \infty} \frac{D(N)}{\sqrt{N \log N}} > 0. \]
    In particular, the proposed bound $D(N) \ll N^{1/2}(\log N)^{o(1)}$ fails for this sequence for almost all $\alpha$. Consequently, the stronger bound $D(N) \ll N^{1/2}(\log\log N)^{O(1)}$ also fails.
This was proved by István Berkes and Walter Philipp in \cite{BP94}.

\textbf{References}
\begin{enumerate}
    \item \cite{BP94} I. Berkes and W. Philipp, ``The size of trigonometric and Walsh series and uniform distribution mod 1'', \textit{Journal of the London Mathematical Society} (2) \textbf{50} (1994), 454–464.
\end{enumerate}

\end{solution}

\subsection{Erdős-1105}

\begin{problem}{Erdős-1105 \cite{ESS75}}
    The anti-Ramsey number $\mathrm{AR}(n,G)$ is the maximum possible number of colours in which the edges of $K_n$ can be coloured without creating a rainbow copy of $G$ (i.e. one in which all edges have different colours). Let $C_k$ be the cycle on $k$ vertices. Is it true that\[\mathrm{AR}(n,C_k)=\left(\frac{k-2}{2}+\frac{1}{k-1}\right)n+O(1)?\]Let $P_k$ be the path on $k$ vertices and $\ell=\lfloor\frac{k-1}{2}\rfloor$. If $n\geq k\geq 5$ then is $\mathrm{AR}(n,P_k)$ equal to\[\max\left(\binom{k-2}{2}+1, \binom{\ell-1}{2}+(\ell-1)(n-\ell+1)+\epsilon\right)\]where $\epsilon=1$ if $k$ is odd and $\epsilon=2$ otherwise?
\end{problem}

\begin{remark}
    The problem on cycles is resolved by Montellano-Ballesteros--Neumann-Lara in \cite{MBNL05}. 
    The problem on paths is resolved by (unpublished) work of Yuan \cite{Yuan21}, which is already noted on Bloom's webpage. {\Aletheia}'s response amounts to nothing more than citing and summarizing these papers, so below we trim it significantly. {\Aletheia} hallucinates publication information for \cite{Yuan21}, which is in fact still unpublished.
\end{remark}

\begin{solution}{Solution to Erdős-1105}

\begin{theorem}
    For any integer $k \ge 3$, the anti-Ramsey number of the cycle $C_k$ satisfies
    \[ \mathrm{AR}(n, C_k) = \left(\frac{k-2}{2} + \frac{1}{k-1}\right)n + O(1). \]
\end{theorem}

\begin{proof}This follows from the exact determination of $\mathrm{AR}(n, C_k)$ by Montellano-Ballesteros and Neumann-Lara \cite{MBNL05}.
\end{proof}

\begin{theorem}
    Let $P_k$ be a path on $k$ vertices and $\ell = \lfloor \frac{k-1}{2} \rfloor$. For all integers $n \ge k \ge 5$, the anti-Ramsey number $\mathrm{AR}(n, P_k)$ is given exactly by:
    \[ \mathrm{AR}(n, P_k) = \max\left(\binom{k-2}{2}+1, \binom{\ell-1}{2}+(\ell-1)(n-\ell+1)+\epsilon\right), \]
    where $\epsilon=1$ if $k$ is odd and $\epsilon=2$ if $k$ is even.
\end{theorem}

\begin{proof}
This was established by Yuan \cite{Yuan21}. 
\end{proof}

\textbf{References}
\begin{enumerate}
    \item\cite{MBNL05} J. J. Montellano-Ballesteros and V. Neumann-Lara, ``An anti-Ramsey theorem on cycles'',  \textit{Graphs and Combinatorics}, vol.~21, no.~3, pp.~343–354, 2005.
    
    \item\cite{Yuan21} L. Yuan, ``The anti-Ramsey number for paths''.
\end{enumerate}
\end{solution}

\appendix

\section{Erd\H{o}s-75: a case study within a case study}\label{app:75}

The purpose of this Appendix is to document an example of issues that can arise in evaluating a solution \emph{even after correctness is ascertained}. For Erd\H{o}s-75, {\Aletheia} devised a correct solution to a non-trivial problem, by Literature Identification. Moreover, a different internal model, when prompted with Erd\H{o}s-75, devised a solution that solved a strengthening of Erd\H{o}s-75 explicitly asked by Erd\H{o}s in \cite{Erd95} (and stated on \href{https://www.erdosproblems.com/}{\ttfamily ErdosProblems.com}). The proof of this strengthened version was later found to be already in \cite{EHS82}, by essentially the same argument. However, we had difficulty auditing the logs to see if this was a genuine Independent Rediscovery or AI plagiarism. 

The point became moot when, after consulting external experts, we discovered that the problem as listed on \href{https://www.erdosproblems.com/}{\ttfamily ErdosProblems.com} was \emph{not} the intended formulation; even though it was accurately transcribed from a paper \cite{Erd95} of Erd\H{o}s, \emph{Erd\H{o}s's own formulation was itself flawed} there. However, it was accurately transcribed by Erd\H{o}s in 
\cite{EHS82} and \cite{Erd95d}. 

We remark that a similar issue occurred with Erd\H{o}s-124 (before we embarked on this effort), but in that case the flaw in the formulation was more obvious because the solution was a trivial reduction to the literature. For transparency, we document below the autonomous solutions to Erd\H{o}s-75 and its strengthened version as they were originally listed on \href{https://www.erdosproblems.com/}{\ttfamily ErdosProblems.com}. 

\subsection{Erd\H{o}s-75}

\begin{problem}{Erd\H{o}s-75 \cite{EHS82, Erd95, Erd95d}}
    Is there a graph of chromatic number $\aleph_1$ such that for all $\epsilon > 0$ if $n$ is sufficiently large and $H$ is a subgraph on $n$ vertices then $H$ contains an independent set of size $> n^{1-\epsilon}$?
\end{problem}

\begin{remark}
The solution to Erd\H{o}s-75 below is an immediate reduction to work of Lambie-Hanson \cite{LH20}, which is already known to experts. We note that {\Aletheia}'s solution is somewhat overkill; any $f(x)$ whose inverse is at most $x^{o(1)}$ would suffice here, i.e.,~there is no need to take the double exponential. For the purposes of faithfully reproducing its proof, we leave this unchanged, but this acts as a reminder that current AI-generated proofs can often be unnecessarily complicated compared to those of humans.
\end{remark}

\begin{solution}{Solution to Erd\H{o}s-75}

The answer is Yes.

\begin{theorem}
There exists a graph $G$ with chromatic number $\aleph_1$ such that for all $\epsilon > 0$, if $n$ is sufficiently large and $H$ is a subgraph of $G$ on $n$ vertices, then $H$ contains an independent set of size greater than $n^{1-\epsilon}$.
\end{theorem}

\begin{proof}
The proof relies on a result by Chris Lambie-Hanson regarding the existence of uncountably chromatic graphs whose finite subgraphs have arbitrarily slowly growing chromatic numbers relative to their size.

\medskip
\noindent \textbf{Lambie-Hanson's theorem.}
The following theorem was proven in ZFC (Zermelo–Fraenkel set theory with Choice):

\begin{theorem}[Lambie-Hanson \cite{LH20}, 2020]
For every function $f: \mathbb{N} \to \mathbb{N}$, there exists a graph $G$ with chromatic number $\chi(G) = \aleph_1$ such that for every integer $k \ge 3$, every subgraph of $G$ with chromatic number at least $k$ has at least $f(k-3)$ vertices.
\end{theorem}


\medskip
\textbf{Construction of the graph.}
Let $f: \mathbb{N} \to \mathbb{N}$ be the double exponential function defined by:
$$ f(x) = 2^{2^x}. $$
Applying Lambie-Hanson's theorem with this specific function $f$, there exists a graph $G$ such that $\chi(G) = \aleph_1$, and for any subgraph $H \subseteq G$, if $\chi(H) = k \ge 3$, then $|V(H)| \ge 2^{2^{k-3}}$.

\medskip
\textbf{Analysis of finite subgraphs.}
Let $\epsilon > 0$ be an arbitrary positive constant. Consider any subgraph $H$ of $G$ with $n$ vertices, where $n$ is sufficiently large (specifically, large enough to satisfy the inequality derived below).

Let $k = \chi(H)$ be the chromatic number of $H$. If $k < 3$, then $\chi(H) \le 2$ and $\alpha(H) \ge n/2$, which is greater than $n^{1-\epsilon}$ for large $n$. Thus, we assume $k \ge 3$.
According to the construction of $G$, the number of vertices $n$ must satisfy:
$$ n \ge 2^{2^{k-3}}. $$
Taking logarithms base 2 twice yields:
$$ \log_2(\log_2 n) \ge k - 3, $$
which provides an upper bound for the chromatic number of $H$:
$$ \chi(H) \le \log_2(\log_2 n) + 3. $$

We recall the elementary inequality relating the independence number $\alpha(H)$ and chromatic number $\chi(H)$ for a graph with $n$ vertices:
$$ \alpha(H) \ge \frac{n}{\chi(H)}. $$
Substituting our upper bound for $\chi(H)$:
$$ \alpha(H) \ge \frac{n}{\log_2(\log_2 n) + 3}. $$

\medskip
\noindent \textbf{Verification of the Independent Set condition.}
We require that $\alpha(H) > n^{1-\epsilon}$ for all sufficiently large $n$. Comparing our lower bound for $\alpha(H)$ to the desired value:
$$ \frac{\alpha(H)}{n^{1-\epsilon}} \ge \frac{n / (\log_2(\log_2 n) + 3)}{n^{1-\epsilon}} = \frac{n^\epsilon}{\log_2(\log_2 n) + 3}. $$
Because $\epsilon$ is a fixed positive constant, the polynomial term $n^\epsilon$ grows strictly faster than the iterated logarithmic term as $n \to \infty$. Formally:
$$ \lim_{n \to \infty} \frac{n^\epsilon}{\log_2(\log_2 n) + 3} = \infty. $$
Thus, there exists an integer $N_\epsilon$ such that for all $n > N_\epsilon$,
$$ \frac{n^\epsilon}{\log_2(\log_2 n) + 3} > 1 \implies \alpha(H) > n^{1-\epsilon}. $$

This establishes that for the chosen graph $G$, for \emph{any} $\epsilon > 0$, every sufficiently large finite subgraph $H$ contains an independent set of the required size.
\end{proof}

\medskip
\noindent \textbf{Conclusion.}
Yes, such a graph exists. By utilizing Lambie-Hanson's construction in ZFC, we obtain a graph of chromatic number $\aleph_1$ where the chromatic number of finite subgraphs grows so slowly (at most double-logarithmically in vertex count) that the independence number $\alpha(H)$ asymptotically exceeds $n^{1-\epsilon}$ for every $\epsilon > 0$.

\medskip
\textbf{Reference}
\begin{enumerate}
    \item \cite{LH20} C. Lambie-Hanson, ``On the growth rate of chromatic numbers of finite subgraphs'', \textit{Advances in Mathematics}, vol.~369, Art.~107176, 2020.
\end{enumerate}
\end{solution}

  As explained on Bloom's website, \cite{Erd95} suggests a strengthened version where $H$ contains an independent set of size $ \gg n$. We pose this below. 
\begin{problem}{Erd\H{o}s-75-strengthened-mistranscribed}
    Is there a graph of chromatic number $\aleph_1$ such that if $n$ is a positive integer and $H$ is a subgraph on $n$ vertices then $H$ contains an independent set of size $\gg n$?
\end{problem}

We will shortly present the autonomous solution to this problem, which we reiterate is already essentially present in \cite{EHS82} (and not cited in the autonomous solution). However, we first note the \emph{correct} statement should also demand that the graph have $\aleph_1$ vertices. 

\begin{problem}{Erd\H{o}s-75-strengthened-correct}
    Is there a graph of chromatic number \textbf{and cardinality} $\aleph_1$ such that if $n$ is a positive integer and $H$ is a subgraph on $n$ vertices then $H$ contains an independent set of size $\gg n$?
\end{problem}

\begin{remark}
The following solution has been edited from the raw model output with the following three main changes. Firstly, the most mathematical change is that the model's solution did not assume the Continuum Hypothesis (CH), but made an invalid step in the part of the output that corresponds to Lemma \ref{prop: chi ge aleph1} below, which is repaired by imposing CH. Secondly, the model's output had an additional remark with a hallucinated reference to a result of Erd\H{o}s--Hajnal; however, this remark played no role in the proof. Thirdly, we have rearranged the argument in a way that we found more compact and easier to follow. 
\end{remark}

\begin{solution}{Solution to Erd\H{o}s-75-strengthened-mistranscribed} We answer the question affirmatively with the following result:
\begin{theorem} \label{thm: Erdos-75 with linear demand}
Assume the continuum hypothesis. Then there exists a graph $\Gamma$ with $\chi(\Gamma)=\aleph_1$ and such that for each $n \in \mathbb{Z}_{\geq 1}$ and any subgraph $H$ of $\Gamma$ on $n$ vertices, one has an independent subset of $H$ of cardinality at least $\frac{n}{4}$.
\end{theorem}
Let us begin by introducing some notation. For a totally ordered set $D$, we define a graph $\Gamma(D)$ having vertex set the pairs $(x,y)$ with $x<y$ elements of $D$, and linking precisely the elements of the form $(x,y),(y,z)$ with an edge. These graphs are often called \emph{shift graphs} in the literature.

We shall follow von Neumann's definition of ordinal numbers and identify an ordinal number $D$ with its set of predecessors. Let $X$ be an ordinal number. Let $D$ be an ordinal number such that $|D|=|\{0,1\}^X|$. Let $\phi:D \to \{0,1\}^X$ be any choice of a bijection witnessing this equality of cardinality. Observe that given such a $\phi$ and $d$ in $D$, we can view $\phi(d)$ as a function from $X$ to $\{0,1\}$. So for each $\delta$ in $X$ we can evaluate $\phi(d)(\delta)$. Let $d_1<d_2$ be two elements of $D$. Let 
$$c_{\phi}(1)(d_1,d_2):=\min\{\delta \in X: \phi(d_1)(\delta) \neq \phi(d_2)(\delta)\}.
$$
And now place
$$c_{\phi}(2)(d_1,d_2):=\phi(d_1)(c_{\phi}(1)(d_1,d_2)).
$$

We have the following first auxiliary lemma. 
\begin{lemma} \label{prop: chi le aleph1}
Let $X$ be an ordinal number. Let $D$ be an ordinal number such that $|D|=|\{0,1\}^X|$. Let $\phi:D \to \{0,1\}^X$ be any choice of a bijection witnessing this equality of cardinality. Then the function
$$c_{\phi}:=(c_{\phi}(1),c_{\phi}(2)): V(\Gamma(D)) \to X \times \{0,1\}
$$
induces a coloring of $\Gamma(D)$ with at most $|X \times \{0,1\}|$ colors. In particular, if $X$ is infinite then
$$\chi(\Gamma(D)) \leq |X|. 
$$
\begin{proof}
Let us show that $c_{\phi}$ is indeed a coloring of $\Gamma(D)$. Indeed observe that if $(d_1,d_2),(d_3,d_4)$ attain the same color, then we must have in particular a $\delta$ in $X$ such that $\phi(d_2)(\delta) \neq \phi(d_1)(\delta)=\phi(d_3)(\delta)$. This in particular forces $d_2 \neq d_3$. Hence, the pair $(d_1,d_2),(d_3,d_4)$ is not linked, establishing that the map is a coloring. Thus, the desired consequence $\chi(\Gamma(D)) \leq |X|$ follows now by definition of chromatic number and the fact that if $X$ is infinite then $|X|=|X \times \{0,1\}|$.  
\end{proof}
\end{lemma}
Our second auxiliary lemma is as follows. Recall that we are operating under the continuum hypothesis.  
\begin{lemma} \label{prop: chi ge aleph1}
Let $X$ be an ordinal number such that $|X|=|\{0,1\}^{\mathbb{Z}}|=\aleph_1$. Let $D$ be an ordinal number such that $|D|=|\{0,1\}^X|$. Then
$$\chi(\Gamma(D)) \geq \aleph_1. 
$$
\begin{proof}
The main tool is the infinitary generalization of Ramsey's theorem due to Erd\H{o}s and Rado. By the continuum hypothesis, $D$ is strictly larger than $2^{\aleph_0}=\aleph_1$. The Erd\H{o}s--Rado theorem thus implies that, if the edges of the complete graph $K_D$ are colored using $\aleph_0$ many colors, then there is a monochromatic clique of cardinality strictly exceeding $\aleph_0$.

Now, suppose that $\chi(\Gamma(D))<\aleph_1$, so that $\chi(\Gamma(D))\leq\aleph_0$. Let $c\colon \Gamma(D) \to \mathbb{Z}_{\geq 1}$ be such a coloring map. This coloring function induces a function on the edges of the complete graph $K_D$ on vertex set $D$. The Erd\H{o}s--Rado theorem shows now that $D$ contains a subset $C$ with $|C|>|\mathbb{Z}|$ in which all the edges have the same color. Let $x<y<z$ be three distinct elements of $C$. We have that $(x,y)$ and $(y,z)$ have the same color. However, $(x,y)$ and $(y,z)$ are adjacent in $\Gamma(D)$, so this violates the coloring property. This shows that no such coloring exists, yielding the desired conclusion. 
\end{proof}
\end{lemma}
Our third auxiliary lemma deals with the independent sets. 
\begin{lemma} \label{prop: indep set ge n/4}
Let $D$ be any totally ordered set. Let $n$ be in $\mathbb{Z}_{\geq 1}$. Let $H$ be a subgraph of $\Gamma(D)$ whose set of vertices has cardinality $n$. Then $H$ has an independent subset of cardinality at least $\frac{n}{4}$. 
\begin{proof}
Denote by $S$ the total set of coordinates of vertices in $H$. Observe that given a partition $S=A \cup (S-A)$ of $S$, the set $I(A):=\{(x,y) \in V(H): x \in A, y \in S-A\}$ is an independent set. Indeed $(x,y),(y,z)$ being in $I(A)$ would force $y$ to be both in $A$ and in $S-A$. But now pick $A$ uniformly at random as a subset of $S$ and compute
\begin{align*}
    \mathbb{E}[|I(A)|]
    &=\mathbb{E}[\sum_{v \in V(H)}1_{v \in I(A)}]
    =\sum_{v \in V(H)}\mathbb{E}[1_{v \in I(A)}]
    \\&=\sum_{(x,y) \in V(H)}\mathbb{P}[1_{(x,y) \in A \times (S-A)}]
    =\sum_{v \in V(H)}\frac{1}{2}\cdot \frac{1}{2}
    =\frac{n}{4}.
\end{align*}
Since for random $A$ the average value of $|I(A)|$ is $\frac{n}{4}$, there must be a choice of $A$ where $|I(A)|$ is at least $\frac{n}{4}$.     
\end{proof}
\end{lemma}

\begin{proof}[Proof of Theorem \ref{thm: Erdos-75 with linear demand}] Let $X$ be the first uncountable ordinal. We have, by the continuum hypothesis, that $|X|=\aleph_1=|\{0,1\}^{\mathbb{Z}}|$. Let $D$ be the first ordinal of size $|\{0,1\}^X|$. By Lemma \ref{prop: chi ge aleph1}, we know that $\chi(\Gamma(D)) \geq \aleph_1$. On the other hand Lemma \ref{prop: chi le aleph1} forces $\chi(\Gamma(D)) \leq \aleph_1$. Altogether, we have that
$$\chi(\Gamma(D))=\aleph_1.
$$
Finally Lemma \ref{prop: indep set ge n/4} shows that for each $n$ in $\mathbb{Z}_{\geq 1}$ and each subgraph $H$ of $\Gamma(D)$ with $n$ vertices, we must have an independent subset of $H$ of cardinality at least $\frac{n}{4}$. This is precisely the desired conclusion. 
\end{proof}
\end{solution}

\bibliographystyle{amsalpha}
\bibliography{Bibliography}

\end{document}